\documentclass{article}
\usepackage[english]{babel}
\usepackage[margin=1.2in]{geometry}

\usepackage{tikz,tikz-cd}
\usepackage{subcaption,graphicx}
\usepackage{comment}
\usepackage{centernot}
\usepackage{mathtools,amsmath,amssymb,amsthm}
\usepackage{hyperref}
\usepackage{tikzit}			
\usepackage{algorithm}
\newcommand{\no}[1]{#1^{\scriptscriptstyle \bot}} 

\tikzstyle{morphism}=[fill=white, rounded corners=3pt, draw=black, shape=rectangle]
\tikzstyle{generic morphism}=[fill=white, draw=black, shape=rectangle, dashed]
\tikzstyle{small box}=[fill=white, draw=black,rounded corners=3pt, shape=rectangle, minimum width=0.5cm, minimum height=0.5cm]
\tikzstyle{medium box}=[fill=white, draw=black,rounded corners=3pt, shape=rectangle, minimum width=0.5cm, minimum height=0.8cm]
\tikzstyle{large morphism}=[fill=white, draw=black,rounded corners=3pt, shape=rectangle, minimum width=0.5cm, minimum height=1.2cm]
\tikzstyle{bn}=[fill=black, draw=black, shape=circle, inner sep=1.5pt]
\tikzstyle{bw}=[fill=white, draw=black, shape=circle, inner sep=1.5pt]
\tikzstyle{bin}=[fill=white, draw=black, shape=circle, inner sep=0pt]
\tikzstyle{effect2}=[fill=white, draw=black, regular polygon, regular polygon sides=3, minimum width=0.8cm, inner sep=0pt]
\tikzstyle{state}=[fill=white, draw=black, regular polygon, regular polygon sides=3, minimum width=0.8cm, shape border rotate=90, inner sep=0pt, rounded corners=3pt]
\tikzstyle{and}=[fill=white, draw=black, circuit logic US, and gate, minimum width=0.5cm, minimum height=0.8cm]
\tikzstyle{or}=[fill=white, draw=black, circuit logic US, or gate,minimum width=1cm, minimum height=1cm]
\tikzstyle{not}=[fill=white,circuit logic US, not gate, minimum width=1cm, minimum height=1cm]
\tikzstyle{xor}=[fill=white, draw=black, circuit logic US, xor gate, minimum width=0.5cm, minimum height=0.8cm]
\tikzstyle{if}=[trapezium, draw=black, fill=white, minimum width=6pt, minimum height=8pt, rotate=270]
\tikzstyle{coreg}=[draw, fill=white, rounded rectangle, rounded rectangle right arc=none, minimum height=1.2em, minimum width=1.4em, node font={\scriptsize}]
\tikzstyle{reg}=[draw, fill=white, rounded rectangle, rounded rectangle left arc=none, minimum height=.47cm, minimum width=.47cm, node font={\scriptsize}]
\tikzstyle{medium state}=[fill=white, draw=black, regular polygon, regular polygon sides=3, minimum width=1.3cm, inner sep=0pt, shape border rotate=90]
\tikzstyle{large state}=[fill=white, draw=black, regular polygon, regular polygon sides=3, minimum width=2.2cm, shape border rotate=180, inner sep=0pt]
\tikzstyle{wide state}=[fill=white, draw=black, shape=isosceles triangle, minimum width=0.8cm, shape border rotate=270, inner sep=1.4pt, minimum height=0.5cm, isosceles triangle apex angle=80]
\tikzstyle{wn}=[fill=white, draw=black, shape=circle, inner sep=1.5pt]
\tikzstyle{blue morphism}=[fill=white, draw={rgb,255: red,15; green,0; blue,150}, shape=rectangle, text={rgb,255: red,15; green,0; blue,150}, tikzit category=blue]
\tikzstyle{red morphism}=[fill=white, draw={rgb,255: red,150; green,0; blue,2}, shape=rectangle, text={rgb,255: red,150; green,0; blue,2}, tikzit category=red]
\tikzstyle{blue state}=[fill=white, draw={rgb,255: red,15; green,0; blue,150}, shape=circle, regular polygon, regular polygon sides=3, minimum width=0.8cm, shape border rotate=180, inner sep=0pt, text={rgb,255: red,15; green,0; blue,150}, tikzit category=blue]
\tikzstyle{blue node}=[fill={rgb,255: red,15; green,0; blue,150}, draw={rgb,255: red,15; green,0; blue,150}, shape=circle, tikzit category=blue, inner sep=1.5pt]
\tikzstyle{blue}=[text={rgb,255: red,15; green,0; blue,150}, tikzit draw={rgb,255: red,191; green,191; blue,191}, tikzit category=blue, tikzit fill=white, inner sep=0mm]
\tikzstyle{blue wide state}=[fill=white, draw={rgb,255: red,15; green,0; blue,150}, text={rgb,255: red,15; green,0; blue,150}, shape=isosceles triangle, minimum width=0.8cm, shape border rotate=270, inner sep=1.4pt, minimum height=0.5cm, isosceles triangle apex angle=80]
\tikzstyle{red node}=[fill={rgb,255: red,150; green,0; blue,2}, draw={rgb,255: red,150; green,0; blue,2}, shape=circle, inner sep=1.5pt]
\tikzstyle{Purple node}=[fill={rgb,255: red,120; green,0; blue,120}, draw={rgb,255: red,120; green,0; blue,120}, text={rgb,255: red,120; green,0; blue,120}, shape=circle, inner sep=1.5pt]
\tikzstyle{red}=[text={rgb,255: red,150; green,0; blue,2}, inner sep=0mm, tikzit fill=white, tikzit draw={rgb,255: red,191; green,191; blue,191}]
\tikzstyle{purple}=[text={rgb,255: red,150; green,0; blue,150}, inner sep=0mm, tikzit fill=white, tikzit draw={rgb,255: red,191; green,191; blue,191}]
\tikzstyle{white morphism}=[fill=white, draw=white, shape=rectangle, tikzit draw={rgb,255: red,139; green,139; blue,139}]
\tikzstyle{leak morphism}=[fill=white, draw={rgb,255: red,120; green,0; blue,85}, shape=rectangle, text={rgb,255: red,120; green,0; blue,85}, tikzit category=leak]
\tikzstyle{leak}=[text={rgb,255: red,120; green,0; blue,85}, inner sep=0mm, tikzit fill=white, tikzit draw={rgb,255: red,191; green,191; blue,191}, tikzit category=leak]
\tikzstyle{leak node}=[fill={rgb,255: red,120; green,0; blue,85}, draw={rgb,255: red,120; green,0; blue,85}, shape=circle, inner sep=1.5pt, tikzit category=leak]
\tikzstyle{horiz state}=[fill=white, draw=black, regular polygon, regular polygon sides=3, minimum width=1cm, shape border rotate=90, inner sep=0pt]
\tikzstyle{none_90}=[rotate=90]
\tikzstyle{none_-90}=[rotate=-90]

\tikzstyle{arrow}=[->]
\tikzstyle{dashed box}=[-, dashed]
\tikzstyle{blue arrow}=[-, draw={rgb,255: red,15; green,0; blue,150}, tikzit category=blue]
\tikzstyle{red arrow}=[-, draw={rgb,255: red,150; green,0; blue,2}, tikzit category=red]
\tikzstyle{purple arrow}=[->, draw={rgb,255: red,120; green,0; blue,120}, >=stealth, shorten <=2pt, shorten >=2pt]
\tikzstyle{protected purple arrow}=[->, draw={rgb,255: red,120; green,0; blue,120}, >=stealth, shorten <=2pt, shorten >=2pt, preaction={line width=1.8pt, white, draw}]
\tikzstyle{mapsto}=[{|->}]
\tikzstyle{double wire}=[-, draw, line width=0.8pt, white, preaction={-, draw, line width=1.8pt}]
\tikzstyle{curly brace}=[-, draw=none, tikzit draw={rgb,255: red,128; green,0; blue,128}]
\tikzstyle{protected}=[-, preaction={line width=1.8pt,white,draw}]
\tikzstyle{leak arrow}=[-, tikzit draw={rgb,255: red,150; green,0; blue,120}]
\tikzstyle{protected leak arrow}=[-, tikzit draw={rgb,255: red,150; green,0; blue,120}]
\tikzstyle{hollow arrow}=[-, very thin, white, preaction={line width=0.7pt,draw={rgb,255: red,120; green,0; blue,85}}, tikzit category=leak, tikzit draw={rgb,255: red,150; green,0; blue,120}]
\tikzstyle{protected hollow arrow}=[-, very thin, white, preaction={line width=0.7pt,draw={rgb,255: red,120; green,0; blue,85},preaction={line width=2.1pt,white,draw}}, tikzit category=leak, tikzit draw={rgb,255: red,150; green,0; blue,120}]
\tikzstyle{over arrow}=[-, black, preaction={draw=white, double}]
\tikzstyle{curly brace}=[-, decorate, decoration={brace,amplitude=5pt}]
\tikzstyle{inv curly brace}=[-, decorate, decoration={brace,amplitude=5pt,mirror}]

\tikzstyle{d-wire1 plate}=[-, double=red!20!white]
\tikzstyle{d-wire2 plate}=[-, double=red!32!white]
\tikzstyle{dotted_plate}=[-,rounded corners=3pt, densely dotted, draw=blue, fill opacity=0.4, fill=blue!50!white]
\tikzstyle{twire1}=[-, draw, line width=0.4pt, preaction={-, draw, line width=1.4pt, red!20!white, preaction={-, draw, line width=2.2pt}}]
\tikzstyle{twire2}=[-, draw, line width=0.4pt, preaction={-, draw, line width=1.4pt, red!32!white, preaction={-, draw, line width=2.2pt}}]
\usetikzlibrary{external,arrows,decorations.pathreplacing,decorations.pathmorphing,calligraphy,calc,positioning,fit,arrows.meta,decorations.markings}	
\usepackage{bm}				
\usepackage{pifont}
\usepackage{enumitem} 
	\setlist[enumerate]{label=(\roman*)}  
	\setlist[enumerate,2]{label=(\alph*)} 
\setlist{nolistsep} 
\setlist[itemize]{label={\color{gray} \rotatebox[origin=c]{-90}{\ding{122}}}}

\usepackage{todonotes}

\usepackage[capitalise]{cleveref}
\crefname{app}{Appendix}{Appendices}

\usepackage{titletoc}
\titlecontents{section}[0em]{}{\thecontentslabel{.\hspace{6pt}}}{\hspace*{0em}}{\titlerule*[0.5pc]{.}\contentspage}

\usepackage[outermarks]{titlesec}

\titleformat{\section}
{\bfseries\Large}
{{\thesection}. }
{0ex}
{}
[]

\numberwithin{equation}{section}
\newtheorem{theorem}[equation]{Theorem}
\newtheorem{proposition}[equation]{Proposition}
\newtheorem{lemma}[equation]{Lemma}

\newtheorem{assumption}[equation]{Assumption}
\theoremstyle{definition}
\newtheorem{definition}[equation]{Definition}
\newtheorem{notation}[equation]{Notation}
\newtheorem{example}[equation]{Example}

\newtheorem{remark}[equation]{Remark}

\newtheorem{setting}[equation]{Setting}

\makeatletter
\def\cref@thmoptarg[#1]#2#3#4{%
	    \ifhmode\unskip\unskip\par\fi%
	    \normalfont%
	    \trivlist%
	    \let\thmheadnl\relax%
	    \let\thm@swap\@gobble%
	    \thm@notefont{\fontseries\mddefault\upshape}%
	    \thm@headpunct{.}
	    \thm@headsep 5\p@ plus\p@ minus\p@\relax%
	    \thm@space@setup%
	    #2
	    \@topsep \thm@preskip               
	    \@topsepadd \thm@postskip           
	    \def\@tempa{#3}\ifx\@empty\@tempa%
	      \def\@tempa{\@oparg{\@begintheorem{#4}{}}[]}%
	    \else%
	      \refstepcounter[#1]{#3}
	      \@namedef{cref@#3@alias}{#1}
	      \def\@tempa{\@oparg{\@begintheorem{#4}{\csname the#3\endcsname}}[]}%
	    \fi%
	    \@tempa}%
\makeatother

\newcommand{\newterm}[1]{\emph{\textbf{#1}}}

\newcommand{\abs}[1]{\left\lvert #1 \right\rvert} 
\renewcommand{\emptyset}{\varnothing} 
\newcommand{\longmapsfrom}{\mathrel{\reflectbox{$\longmapsto$}}}

\newcommand{\cat}[1]{{\mathsf{#1}}}

\newcommand{\id}{\mathrm{id}} 		
\newcommand{\tot}{\mathrm{tot}}
\newcommand{\tensor}{\otimes}

\newcommand{\comp}{ 		
	\mathchoice{\,}{\,}{}{} 	
}

\DeclareMathOperator{\cop}{copy}
\DeclareMathOperator{\del}{del}

\newcommand{\copycomp}[1]{\operatorname{Cpy}(#1)}
\newcommand{\compcomp}[1]{\operatorname{Cmp}(#1)}
\newcommand{\In}[1]{\operatorname{In}_{#1}}
\newcommand{\Out}[1]{\operatorname{Out}_{#1}}
\newcommand{\cC}{\mathsf{C}}		
\newcommand{\cD}{\mathsf{D}}		
\newcommand{\graph}[1]{\operatorname{gr}(#1)}
\newcommand{\freehyp}[1]{\mathsf{FreeHyp}(#1)}
\newcommand{\freecdo}[1]{\mathsf{FreeCD}(#1)}
\newcommand{\vcat}[1]{\mathsf{VL}_{#1}}
\newcommand{\syn}[1]{\mathsf{HSyn}_{#1}} 
\newcommand{\cdsyn}[1]{\mathsf{CDSyn}_{#1}}
\newcommand{\cartsyn}[1]{\mathsf{FreeC}_{#1}} 
\newcommand{\parents}[1]{\operatorname{Pa}(#1)}
\newcommand{\mor}[1]{\operatorname{Mor}(#1)}
\newcommand{\tr}[1]{\operatorname{Tr}(#1)}
\newcommand{\ctr}[1]{\operatorname{Tr}_C(#1)}
\newcommand{\cmor}[1]{\operatorname{Mor}_C(#1)}
\newcommand{\ve}[1]{\operatorname{VE}(#1)}

\newcommand{\h}{\mathcal{H}}
\newcommand{\g}{\mathcal{G}}

\newcommand{\finstoch}{\mathsf{FinStoch}}

\newcommand{\odag}{\cat{ODAG}}
\newcommand{\ougr}{\cat{OUGr}}
\newcommand{\cdcat}{\cat{CDCat}}
\newcommand{\hypcat}{\cat{HypCat}}
\newcommand{\bn}{\cat{BN}}
\newcommand{\mn}{\cat{MN}}
\newcommand{\cn}{\cat{CN}}
\newcommand{\finprojstoch}{\mathsf{FinProjStoch}}

\newcommand{\mat}[1][\mathbb{R}^{\geq 0}]{\mathsf{Mat}(#1)}

\newcommand{\as}[1]{
	\def\relstate{#1}%
	\ifx\relstate\empty
		\text{a.s.}%
	\else
		{#1\text{-a.s.}}%
	\fi
}

\newcommand{\normalisation}[1]{\mathtt{n}(#1)}
\newcommand{\clique}[1]{C\ell(#1)}

\providecommand{\given}{\,|\,}			
\newcommand{\edge}{\mathrel{\rule[0.5ex]{1em}{0.2pt}}}

\title{Bayesian Networks, Markov Networks, Moralisation, Triangulation: a Categorical Perspective}

\author{Antonio Lorenzin and Fabio Zanasi}

\begin{document}

\maketitle

\begin{abstract}
	\noindent Moralisation and Triangulation are transformations allowing to switch between different ways of factoring a probability distribution into a graphical model. Moralisation allows to view a Bayesian network (a directed model) as a Markov network (an undirected model), whereas triangulation addresses the opposite direction. 
	We present a categorical framework where these transformations are modelled as functors between a category of Bayesian networks and one of Markov networks. 
	The two kinds of network (the objects of these categories) are themselves represented as functors from a `syntax' domain to a `semantics' codomain.
	Notably, moralisation and triangulation can be defined inductively on such syntax via functor pre-composition. 
	Moreover, while moralisation is fully syntactic, triangulation relies on semantics.
	This leads to a discussion of the variable elimination algorithm, reinterpreted here as a functor in its own right, that splits the triangulation procedure in two: one purely syntactic, the other purely semantic. 
	This approach introduces a functorial perspective into the theory of probabilistic graphical models, which highlights the distinctions between syntactic and semantic modifications.
\end{abstract}
\tableofcontents
\section{Introduction}
In recent years, there has been growing interest in categorical approaches to probabilistic computation. This trend spans several domains, including probabilistic programming, Bayesian inference, and machine learning, see e.g.~\cite{Barthe_Katoen_Silva_2020,lorenz2023causalmodels,shiebler2021categorytheorymachinelearning} for an overview. 
The appeal of these approaches lies in their capacity to provide rigorous and principled semantics, often surpassing the formal clarity offered by traditional methods.
One compelling example is the study in \cite{jacobs2019mathematics}, where categorical structures elucidate the conceptual differences between two widely adopted update rules for reasoning with soft evidence.
Moreover, the emphasis on abstraction offered by category theory enables a unified treatment of different probabilistic frameworks. 
A prominent formalism in this respect is that of \emph{Markov categories}~\cite{fritz2019synthetic}.
This provides a general setting in which fundamental probabilistic notions --- such as conditioning, independence, and determinism --- can be formulated, and moreover the instantiation of such notions correspond to the expected meaning in standard frameworks (e.g., discrete, measure-theoretic and Gaussian models).
Another important aspect of the categorical paradigm is its accent on compositionality, as it redirects our focus to decomposing probabilistic computations into more fundamental building blocks.

In this context, we aim to investigate probabilistic graphical models (PGMs).
PGMs are graphical representations of probability distributions, widely used in machine learning, statistics, and artificial intelligence.
The structure underlying a PGM is combinatorial, typically given by a graph in which vertices represent random variables and edges describe probabilistic relationships. 
More precisely, the absence of an edge is generally interpreted as a conditional independence between the corresponding variables.
The distinction between directed and undirected graphs leads to two different classes of PGMs, which capture causality and correlation respectively.
The two approaches are formalised by Bayesian networks, based on directed acyclic graphs (DAGs), and Markov networks (also called Markov random fields), based on undirected graphs.

To bridge the gap between them, we can translate one representation into the other. 
This is achieved via \emph{moralisation}, which transforms a Bayesian network into a Markov network, and \emph{triangulation}, which handles the opposite direction.
Importantly, these transformations must preserve the information offered by the combinatorial structure --- formally, they cannot introduce new conditional independencies. 
On the other hand, the conceptual distinction between correlation and causation may result in the loss of certain conditional independencies during conversion.
Moralisation and triangulation are especially relevant in exact inference tasks, such as the junction tree algorithm, clique tree message passing,  variable elimination, and graph-based optimisation; see~\cite{barberBRML2012,koller2009probabilistic} for an overview.

In the traditional treatment of PGMs, the combinatorial structure is rarely powerful enough to stand independently of its probabilistic content, and one is often forced to consider the associated probability distribution explicitly.
Formal tools for reasoning about the combinatorics are limited, and sometimes more convoluted than necessary (compare, for example, traditional d-separation~\cite[Sec.~3.3]{koller2009probabilistic} with its categorical counterpart~\cite{fritz2022dseparation}).
Moreover, this interplay enters crucially into the definitions of moralisation and triangulation: 
Although these may appear to be purely combinatorial operations, they must be applied to the whole network, including its probabilistic component.

The present paper aims to study Bayesian and Markov networks, along with the transformations of moralisation and triangulation, from the perspective of categorical semantics.
This is inspired by Lawvere's well-established approach to functorial semantics~\cite{lawvere1963functorial}, wherein an algebraic theory is described as a freely generated cartesian category of terms $\cartsyn{T}$, and a model of $T$ as a product-preserving functor $\cartsyn{T} \to \cat{C}$, with $\cat{C}$ a cartesian category. 
Through this lens, $\cartsyn{T}$ is interpreted as the syntax, while $\cat{C}$ plays the role of semantics. 
In categorical approaches to Bayesian networks, initiated in~\cite{fong2013causaltheoriescategoricalperspective,JacobsZ16,jacobs2019causal_surgery}, cartesianity turns out to be too strong, prompting the need for a weaker structure. This results in the formalism of \emph{copy-delete (CD) categories}~\cite{CorradiniG99,fong2013causaltheoriescategoricalperspective,JacobsZ16,Fritz_2023}, which are expressive enough to model the connectivity of DAGs, and yet weak enough to comprise standard probability settings.
Among the many contributions, \cite{jacobs2019causal_surgery} is particularly relevant for our purposes, as it establishes a bijective correspondence between Bayesian networks on a DAG $\g$ and structure-preserving functors from $\cdsyn{\g}$, the freely generated CD-category obtained from $\g$, to $\finstoch$, the CD-category of finite sets and stochastic matrices.
This reformulation mirrors Lawvere's approach, so we can reasonably regard Bayesian networks over $\g$ as models of the `theory' $\g$ in $\finstoch$. 
In particular, we can distinguish the syntax, given by $\g$, from its semantic content, formally captured by the functor.
Moreover, morphisms of $\cdsyn{\g}$ are expressed via \emph{string diagrams}~\cite{selinger11graphical,piedeleuzanasi}, highlighting their role as a two-dimensional syntax for graphs.
Surprisingly, this standard tool in category theory literature closely resembles factor graphs, a well-known type of PGM~\cite{loeliger2004factorgraphs}.

Similarly to the case of Bayesian networks, here we provide a functorial semantics for Markov networks. 
Using the formalism of \emph{hypergraph categories}~\cite{fong2019hypergraph}, we show that Markov networks on an undirected graph $\h$ correspond bijectively to structure-preserving functors from $\syn{\h}$, the freely generated hypergraph category obtained from $\h$, to $\mat$, the hypergraph category of finite sets and matrices.
Afterwards, we develop the necessary ingrendients to discuss moralisation and triangulation as functors.
After removing redundant information from the networks to better focus on conditional independencies, we define the categories of Bayesian and Markov networks, denoted by $\bn$ and $\mn$ respectively, where morphisms can be decomposed into syntactic and semantic components, in accordance with our functorial approach.
In the study of transformations, the standard triangulation is translated in two distinct functors to clarify the separation between syntactic and semantic aspects.
This splitting is enabled by the so-called chordal networks, and their category $\cn$: 
Specifically, the triangulation functor $\tr{-}$ factors as $\mn \to \cn \to \bn$, where $\ctr{-}\colon \mn \to \cn$ is completely syntactic, and $\ve{-}\colon \cn \to \bn$ is semantic. 
The latter is closely related to the Variable Elimination algorithm, hence the name.
A key feature of our functorial perspective is that both moralisation and triangulation are simply expressed as functor precomposition, freely defined on the syntactic generators of $\cdsyn{\g}$ and $\syn{\h}$. As a major takeaway, the traditional combinatorial viewpoint on PGMs is translated to a purely syntactic one, which facilitates modularisation and algebraic reasoning.

The paper concludes with a study of the functorial interplay, culminating in the following commutative diagram
\begin{equation}\label{eq:interplay_intro}
\begin{tikzcd}[column sep=large]
	&\bn \ar[r,"\mor{-}"]\ar[rrr, bend left=20, start anchor=north east, "\tr{\mor{-}}"] & \mn \ar[rr, color=red, "\tr{-}"]\ar[rd,"\ctr{-}" below left] && \bn\\
	\cn \ar[ru, color=red, "\ve{-}"]\ar[rru,"\cmor{-}" below right]\ar[rrr,equal] &&&\cn\ar[ru, color=red, "\ve{-}"] &
\end{tikzcd}
\end{equation}
where only the red arrows denote functors that require semantic assumptions. $\mor{-}$ and $\cmor{-}$ are moralisation and its chordal analogue, respectively, and similarly for triangulation $\tr{-}$.
We also show that $\mor{-}$ and $\tr{-}$ do not form an adjunction --- conceptually, this would require introducing additional conditional independencies, thereby adding information that was not originally present.

\paragraph{Conference version}
A preliminary version of this work was presented at CALCO 2025, under the title ``An Algebraic Approach to Moralisation and Triangulation of Probabilistic Graphical Models'' \cite{lorenzin2025moralisation}. 
The present journal version significantly extends that prior work, in particular:
\begin{itemize}
	\item We refine the notion of irredundant networks in \cref{sec:irredundant_networks} to better align with its traditional meaning, leading to a revised definition of the categories of networks;
	\item We decompose triangulation into two steps, separating syntactic modifications from a semantic “check”;
	\item We comment on the junction tree algorithm (\cref{rem:JTA}) and the challenges behind a possible categorical semantics approach;
	\item We relate our discussion to the variable elimination algorithm, which is closely connected to the functor $\ve{-}$ (see \cref{sec:VE}), and acts as an intermediate step in the description of triangulation.
\end{itemize}

\paragraph{Outline}
The necessary categorical notions for the present purposes are introduced in \cref{sec:cd-hypergraphcats}. In particular, the conditions required on semantics are specified in \cref{set:conditionals}.
\cref{sec:networks} begins our treatment of categorical semantics by showing how Bayesian and Markov networks can be viewed as functors (\cref{prop:bayesian_functor,prop:markov_functor}). 
The proposed syntax exhibits important functorial properties, covered in \cref{sec:syntax_functoriality} (\cref{thm:graphhom-cd,thm:graphhom-hyp}).
Since moralisation and triangulation are primarily concerned with preserving conditional independences, we introduce an irredundant formulation in \cref{sec:irredundant_networks}, yielding a description via functorial factorisations (\cref{prop:irredundantbayesian_functor,prop:irredundantmarkov_functor}). 
\cref{sec:categoriesnetworks} then turns to morphisms between networks, which we discuss in detail, offering insight and motivation for the present choices.
The final sections focus on transformations. 
Following the discussion of moralisation in \cref{sec:moralisation}, triangulation is treated in two parts: \cref{sec:triangulation} covers syntactic modifications, and \cref{sec:VE} covers semantic assumptions.
Especially important is the connection between the latter and the variable elimination algorithm (\cref{rem:ve}), which inspired the functor's and the section's name.
We conclude by presenting a complete picture of the functorial interplay in \cref{sec:interplay}.

\paragraph{Acknowledgements}
We would like to thank Tom\'a\v{s} Gonda, Oliver B\o{}ving, Leo Lobski and Ralph Sarkis for fruitful discussions.
We acknowledge funding from ARIA Safeguarded AI TA1.1 programme.

\section{Copy-Delete and Hypergraph Categories}\label{sec:cd-hypergraphcats}
The fundamental distinction between Bayesian and Markov networks is that one is a \emph{directed} and the other is an \emph{undirected} model. 
In this section we recall the categorical structures necessary to account for this difference. First, we focus on the directed case. Following the usual categorical perspective on Bayesian networks~\cite{fong2013causaltheoriescategoricalperspective,JacobsZ16}, we identify in \emph{copy-discard (CD) categories} the requirements needed to interpret the directed acyclic structures of these models. Intuitively, in CD-categories each object has a `copy' and a `delete' map, expressing the ability of nodes of being connected to multiple edges, or to no edge at all. We assume familiarity with \emph{string diagrams}~\cite{selinger11graphical,piedeleuzanasi}, the graphical language of monoidal categories, which we adopt to emphasise the interpretation of morphisms as graphical models.

\begin{assumption}[Strictness]\label{assump:strict}
	Throughout, monoidal categories and functors are \newterm{strict}, i.e.\ associators, unitors and coherence morphisms for the functors are all identities.
	Additionally, the monoidal product is always denoted by $\otimes$, whereas $I$ indicates the monoidal unit.
\end{assumption}

\begin{definition}\label{def:cdcat}
	A \newterm{CD-category} is a symmetric monoidal category where each object $X$ is equipped with a commutative comonoid respecting the monoidal structure. We write `copy' (comultiplication) and `delete' (counit) maps in string diagram notation, respectively as $\input{copyX.tikz}$ and $\begin{tikzpicture}
	\begin{pgfonlayer}{nodelayer}
		\node [style=bn] (35) at (0.75, 0) {};
		\node [style=none] (36) at (-0.75, 0) {};
		\node [style=none] (37) at (-1.25, 0) {$X$};
	\end{pgfonlayer}
	\begin{pgfonlayer}{edgelayer}
		\draw (35) to (36.center);
	\end{pgfonlayer}
\end{tikzpicture}
$. We omit the object label when unnecessary. The commutative comonoid equations are then displayed as:
		\begin{equation}\label{eq:comonoids}
			{\input{copy_commutative.tikz}}\qquad {\input{copy_associative.tikz}}\qquad {\input{del_unit.tikz}}
		\end{equation}
	Associativity ensures a well-defined `copy' $\input{copymult.tikz}$ with multiple outputs.
	A \newterm{CD-functor} is a symmetric monoidal functor between CD-categories preserving \input{copy.tikz} and \begin{tikzpicture}
	\begin{pgfonlayer}{nodelayer}
		\node [style=bn] (35) at (0.75, 0) {};
		\node [style=none] (36) at (-0.75, 0) {};
	\end{pgfonlayer}
	\begin{pgfonlayer}{edgelayer}
		\draw (35) to (36.center);
	\end{pgfonlayer}
\end{tikzpicture}
. CD-categories and CD-functors form a category $\cdcat$.
\end{definition}

\begin{example}\label{ex:finstoch}
	Our chief example of CD-category is $\finstoch$~\cite[Ex.~2.5]{fritz2019synthetic}, the category whose objects are finite sets\footnote{To ensure strictness of $\finstoch$, we actually take as objects only finite sets whose elements are finite lists, and define the monoidal product $X \otimes Y$ via list concatenation. For example, if $X=\lbrace [a],[b] \rbrace$ and $Y=\lbrace [a],[c] \rbrace$, then $X \otimes Y =\lbrace [a,a], [a,c], [b,a],[b,c] \rbrace$. This caveat is immaterial for our developments.
} and whose morphisms $f\colon X \to Y$ are maps $Y \times X \to \mathbb{R}^{\ge 0}$ such that $\Sigma_{y \in Y} f(y,x) = 1$. We will use a ``conditional notation'' and write $f(y\given x)$ for the image of $(y,x)$ via $f$. 
Note that the requirement on $f$ amounts to saying $f(-\given x)$ is a probability distribution on $Y$ for each $x \in X$. Another way to view $f$ is as a stochastic matrix (i.e., a matrix where each column sums to $1$) with $X$-labelled columns and $Y$-labelled rows. Composition is defined via the Chapman--Kolmogorov equation, or equivalently by product of matrices: given $f \colon X \to Y$ and $g \colon Y \to Z$, $
				g\comp f (z \given x) := \sum_{y \in Y} g(z\given y) f(y \given x)$.
				The tensor product is the Kronecker product of matrices; more explicitly, for $f\colon X \to Y$ and $h \colon Z \to W$, $f \otimes h (y,w \given x,z) := f(y\given x) h(w \given z)$.
				The structural morphisms yielding the CD structure are defined as follows:
				\[
				\begin{array}{ccc}
					\input{copy.tikz} (y,z\given x) \coloneqq \begin{cases}
						1 & \text{if }x=y=z\\
						0 & \text{otherwise}
					\end{cases},&\text{and}&
				\begin{tikzpicture}
	\begin{pgfonlayer}{nodelayer}
		\node [style=bn] (35) at (0.75, 0) {};
		\node [style=none] (36) at (-0.75, 0) {};
	\end{pgfonlayer}
	\begin{pgfonlayer}{edgelayer}
		\draw (35) to (36.center);
	\end{pgfonlayer}
\end{tikzpicture}
 (\given x) \coloneqq 1.
				\end{array}
				\]
				This definition motivates the names `copy' and `delete' for the comonoid operations.
\end{example}


\begin{example}[Free CD-categories]\label{ex:freecdo}
Recall that a \emph{monoidal signature} is a pair $(A, \Sigma)$ consisting of a set $A$ of generating objects and a set $\Sigma$ of generating morphisms typed in $A^{\star}$ (finite lists of $A$-elements). One may freely construct the CD-category associated with $(A,\Sigma)$, denoted $\freecdo{A,\Sigma}$: its set of objects is $A^{\star}$ and morphisms are obtained by combining $\Sigma$-generators with $\input{copyX.tikz}$ and $\begin{tikzpicture}
	\begin{pgfonlayer}{nodelayer}
		\node [style=bn] (35) at (0.75, 0) {};
		\node [style=none] (36) at (-0.75, 0) {};
		\node [style=none] (37) at (-1.25, 0) {$X$};
	\end{pgfonlayer}
	\begin{pgfonlayer}{edgelayer}
		\draw (35) to (36.center);
	\end{pgfonlayer}
\end{tikzpicture}
$, for each $X \in A$, via sequential and parallel composition, and then quotienting by~\eqref{eq:comonoids} and the laws of symmetric strict monoidal categories. We refer, e.g., to \cite{baez2018props,bonchi2018deconstructing,zanasi2018} for details.
\end{example}


Before proceeding, we recall a nontrivial property that conceptually distinguishes syntax from semantics. 
In this work, such a distinction will play an active role in \cref{thm:triangulation_finstoch} (see also \cref{prop:pearlupdate}).
\begin{definition}\label{def:conditionals}
	A CD-category \newterm{has conditionals} if for each $f \colon A \to X \otimes Y$, there exists $g\colon A \otimes X \to Y$ such that
	\begin{equation}\label{eq:conditionals}	
		\input{conditional.tikz}
	\end{equation}
	Given a morphism $f \colon A \to X$, its \newterm{normalisation} $\normalisation{f}$ is a choice of a conditional in the following sense: $\input{fAX.tikz}=\input{normalisation.tikz}$.
\end{definition}

We now turn attention to undirected models. We argue the suitable categorical structure is an extension of CD-categories, variously called \emph{hypergraph categories}~\cite{fong2019hypergraph} or \emph{well-supported compact closed categories}~\cite{CARBONI198711}.


\begin{definition}\label{def:hypcat}
	A \newterm{hypergraph category} is a CD-category where furthermore each object $X$ is equipped with a commutative monoid respecting the monoidal structure, and interacting with the comonoid on $X$ via the laws of special Frobenius algebras. We write `compare' (multiplication) and `omni' (unit) maps in string diagram notation, respectively as $\input{compareX.tikz}$ and $\begin{tikzpicture}
	\begin{pgfonlayer}{nodelayer}
		\node [style=bn] (24) at (-0.75, 0) {};
		\node [style=none] (25) at (0.75, 0) {};
		\node [style=none] (26) at (1.25, 0) {$X$};
	\end{pgfonlayer}
	\begin{pgfonlayer}{edgelayer}
		\draw (24) to (25.center);
	\end{pgfonlayer}
\end{tikzpicture}
$. The commutative monoid equations and special Frobenius equations are displayed as:
\begin{equation}\label{eq:hypergraph}
    \begin{aligned}
			{\input{compare_commutativity.tikz}}\qquad {\input{compare_associativity.tikz}}\qquad {\input{omni_unit.tikz}} \\[1em]
						{\input{frobenius_special.tikz}}\qquad {\input{frobenius_snake.tikz}} \qquad\qquad \qquad
   \end{aligned}
\end{equation}
Associativity ensures a well-defined `compare' $\input{comparemult.tikz}$ with multiple inputs.
 A \newterm{hypergraph functor} is a CD-functor preserving the monoid structure. Hypergraph categories and hypergraph functors form a category $\hypcat$.
\end{definition}

\begin{notation}
	We will use the `cups' and `caps' notation: $\begin{tikzpicture}
	\begin{pgfonlayer}{nodelayer}
		\node [style=none] (37) at (-3.25, 0) {};
		\node [style=none] (38) at (-2, -0.75) {};
		\node [style=none] (39) at (-2, 0.75) {};
	\end{pgfonlayer}
	\begin{pgfonlayer}{edgelayer}
		\draw [in=-90, out=180] (38.center) to (37.center);
		\draw [in=180, out=90] (37.center) to (39.center);
	\end{pgfonlayer}
\end{tikzpicture}
\coloneqq \input{cup_def.tikz}$ and $\begin{tikzpicture}
	\begin{pgfonlayer}{nodelayer}
		\node [style=none] (29) at (-2, 0) {};
		\node [style=none] (30) at (-3.25, 0.75) {};
		\node [style=none] (31) at (-3.25, -0.75) {};
	\end{pgfonlayer}
	\begin{pgfonlayer}{edgelayer}
		\draw [in=90, out=0] (30.center) to (29.center);
		\draw [in=0, out=-90] (29.center) to (31.center);
	\end{pgfonlayer}
\end{tikzpicture}
\coloneqq \input{cap_def.tikz}$.
\end{notation}
	
The suitability of hypergraph categories to express undirected models may be better appreciated by observing that there is a bijective correspondence between homsets $[X,Y]$, $[I, X \otimes Y]$, and $[X \otimes Y,I]$ for any objects $X, Y$. These correspondences, obtained by `cups' and `caps', may be seen graphically as `bending wires', or `turning inputs into outputs' and vice versa: \input{fXY.tikz}\, \input{fcapXY.tikz} \, \input{fcupXY.tikz}.
We refer to~\cite{fong2019hypergraph} for a more systematic discussion. 
Additionally, hypergraph categories allow for the decomposition of `cups' and `caps' into the more elementary `copy', `delete', `compare', and `omni' maps.
The expressivity provided by these maps is actually crucial for modelling Bayesian and Markov networks, as all occurring variables are explicitly represented, ensuring that no information is hidden (cf.\ \cref{fig:ex_bayesian_network,fig:markov_network_ex}; see also \cite[Sec.~3]{jacobs2019causal_surgery}).
	
	\begin{example} \label{ex:main}
Our leading example of hypergraph category is $\mat$, the category whose objects are finite sets\footnote{The same caveat as in the previous footnote applies.} and whose morphisms $f\colon X \to Y$ are maps $Y \times X \to \mathbb{R}^{\ge 0}$. Sequential and parallel composition are defined as in $\finstoch$ (\Cref{ex:finstoch}), and we may use analogous conditional notation and matrix representation for its morphisms. In fact, we may regard $\mat$-morphisms as (generic) matrices with entries in $\mathbb{R}^{\ge 0}$. $\finstoch$ is the full subcategory where we restrict to the stochastic matrices. 
The compare-omni morphisms are defined as follows:
				\[
				\begin{array}{ccc}
					\input{compare.tikz} (z\given x,y) \coloneqq \begin{cases}
						1 & \text{if }x=y=z\\
						0 & \text{otherwise}
					\end{cases}, &\qquad& \begin{tikzpicture}
	\begin{pgfonlayer}{nodelayer}
		\node [style=bn] (24) at (-0.75, 0) {};
		\node [style=none] (25) at (0.75, 0) {};
	\end{pgfonlayer}
	\begin{pgfonlayer}{edgelayer}
		\draw (24) to (25.center);
	\end{pgfonlayer}
\end{tikzpicture}
 (x \given ) \coloneqq 1, 
				\end{array}
				\]
			Note that $\begin{tikzpicture}
	\begin{pgfonlayer}{nodelayer}
		\node [style=bn] (35) at (0.75, 0) {};
		\node [style=bn] (36) at (-0.75, 0) {};
	\end{pgfonlayer}
	\begin{pgfonlayer}{edgelayer}
		\draw (35) to (36);
	\end{pgfonlayer}
\end{tikzpicture}
 = \id_I$ is not necessarily satisfied.
	\end{example}

	\begin{example}\label{ex:finprojstoch}
		The category $\finprojstoch$, as defined in \cite[Def.~6.3]{stein2024exactconditions}, is a quotient of $\mat$ given by setting an equivalence relation $f_1 \propto f_2$ whenever there exists $\lambda\in \mathbb{R}^{> 0}$ such that $\lambda f_1(y\given x) = f_2(y\given x)$ for all $x \in X$ and $y \in Y$. Intuitively, $f_1$ and $f_2$ only differ by a global nonzero multiplicative factor.
		$\finprojstoch$ inherits the hypergraph category structure of $\mat$ via the functor $\mat \to \finprojstoch$ sending $f$ to its equivalence class $[f]_{\propto}$. 
		Also, $\finstoch$ embeds in $\finprojstoch$ via the analogously defined functor $f\mapsto [f]_{\propto}$.
	\end{example}
	
\begin{remark}[The Normalisation Cospan]\label{rem:normalisation_cospan}
	The introduction of $\finprojstoch$ is justified by the co\-span 
	$\mat\xrightarrow{q} \finprojstoch \xleftarrow{i} \finstoch$, 
	where $q$ is the quotient and $i$ is the embedding described in \Cref{ex:finprojstoch}. This picture is central to our characterisation of Markov networks, as it describes a normalisation procedure (see~\cref{prop:irredundantmarkov_functor}). 
\end{remark}





This compels us to give the following definition, with the idea to comprise our situation in a more categorical framework. Here, ``tentative'' indicates that the definition is not final, but it is sufficient for our present purposes.
\begin{definition}[Tentative]
Let $\cC$ be a hypergraph category, and let $\cC_{\tot}$ be its subcategory of \newterm{total} morphisms (i.e., those $f$ such that $\begin{tikzpicture}
	\begin{pgfonlayer}{nodelayer}
		\node [style=morphism] (54) at (0, 0) {$f$};
		\node [style=none] (55) at (-2, 0) {};
		\node [style=bn] (58) at (2, 0) {};
		\node [style=none] (59) at (0, 0) {};
	\end{pgfonlayer}
	\begin{pgfonlayer}{edgelayer}
		\draw (55.center) to (54);
		\draw (59.center) to (58);
	\end{pgfonlayer}
\end{tikzpicture}
=\begin{tikzpicture}
	\begin{pgfonlayer}{nodelayer}
		\node [style=bn] (35) at (0.75, 0) {};
		\node [style=none] (36) at (-0.75, 0) {};
	\end{pgfonlayer}
	\begin{pgfonlayer}{edgelayer}
		\draw (35) to (36.center);
	\end{pgfonlayer}
\end{tikzpicture}
$). A \newterm{normalisation cospan} is a hypergraph functor $q \colon \cC \to \cD$ such that the composition $\cC_{\tot}\hookrightarrow \cC \to \cD$ is faithful, and we write it as 
\[
\begin{tikzcd}
	\cC \ar[r,"q"] & \cD &\cC_{\tot} \ar[l,"i" above]
\end{tikzcd}
\]
\end{definition}

\begin{setting}\label{set:conditionals}
The content of this paper can be applied to any hypergraph category $\cC$ with a normalisation cospan and total conditionals, i.e.\ for every morphism $f \colon A \to X \otimes Y$ there exists a $g$ satisfying \eqref{eq:conditionals} that is total.
For the sake of simplicity, we restrict our discussion to $\mat$ throughout, in order to highlight the conceptual motivations behind our choices.
\end{setting}

That $\mat$ does satisfy the wanted property on conditionals is a direct check: given $f\colon A \to X \otimes Y$, then
\[
g(y\mid x,a) \coloneqq \begin{cases}
	\frac{f(x,y\mid a) }{\sum_{y'} f(x,y'\mid a)} & \text{if }\sum_{y'} f(x,y'\mid a)\neq 0 \\
	\frac{1}{\abs{Y}} & \text{if }\sum_{y'} f(x,y'\mid a)= 0,
\end{cases}
\]
where $\abs{Y}$ indicates the cardinality of $Y$, is a total conditional.
	
	\begin{example}[Free Hypergraph Categories]\label{ex:freecor} Analogously to the case of CD-categories (\Cref{ex:freecdo}), one may freely construct a hypergraph category from a monoidal signature $(A,\Sigma)$: the only difference is that the construction of morphisms also involves `structural' generators $\input{compareX.tikz}$ and $\begin{tikzpicture}
	\begin{pgfonlayer}{nodelayer}
		\node [style=bn] (24) at (-0.75, 0) {};
		\node [style=none] (25) at (0.75, 0) {};
		\node [style=none] (26) at (1.25, 0) {$X$};
	\end{pgfonlayer}
	\begin{pgfonlayer}{edgelayer}
		\draw (24) to (25.center);
	\end{pgfonlayer}
\end{tikzpicture}
$ for each $X \in A$, and we additionally quotient by~\eqref{eq:hypergraph}.
	\end{example}

\section{Probabilistic Graphical Models as Functors}~\label{sec:networks}

In this section we characterise both Bayesian networks and Markov networks via the paradigm of functorial semantics. For Bayesian networks, we build on a previous result~\cite[Prop.~3.1]{jacobs2019causal_surgery}, whereas the characterisation for Markov networks is novel. These results also help us to characterise an `irredundant' version of the networks, which is instrumental for the developments in the next sections.

Throughout, all graphs are \newterm{finite}. Also, in a {DAG} (directed acyclic graph) $\g=(V_{\g},E_{\g})$, we denote by $\parents{v}$ the set of parents of a vertex $v$, i.e.~the set of vertices $w$ such that $w \to v$.

\begin{definition}\label{def:BN} A \newterm{Bayesian network} over $\g$ is given by an assignment $\tau(v)$ of a finite set for any vertex $v \in V_{\g}$ together with a stochastic matrix $f_v \colon \tau(v)\times \prod_{w \in\parents{v}} \tau(w)\to [0,1]$ interpreted as a conditional distribution of $v$ knowing $\parents{v}$.
Any such Bayesian network has an associated distribution given by $\omega(V_{\g})= \prod_{v \in V_{\g}} f_v (v \given \parents{v})$.
\end{definition}

\begin{example}\label{ex:BEAR}
The leftmost picture in \cref{fig:ex_bayesian_network} depicts a Bayesian network, borrowed from \cite[Ex.~9.14]{jacobs2021logical}. 
	Vertices represent an alarm ($A=\lbrace a,\no{a}\rbrace$), which may be activated by a burglary ($B=\lbrace b,\no{b}\rbrace$) or an earthquake ($E=\lbrace e, \no{e}\rbrace$). The radio ($R=\lbrace r,\no{r}\rbrace$) reliably reports any earthquake. The two elements of each set represent if the event occurred (e.g., $a$), or not (e.g., $\no{a}$). Edges indicate causal relationship between events, with probabilities given by stochastic matrices (represented as conditional probability tables).
	\begin{figure}[tp]
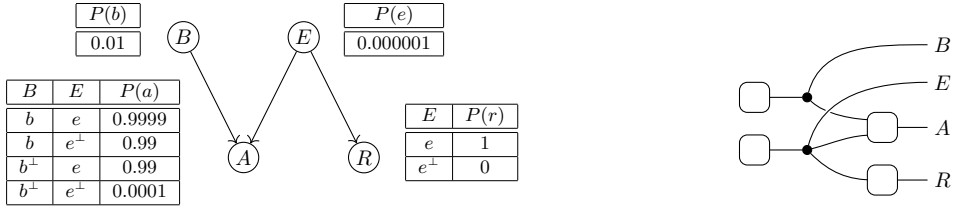

	\centering
	\begin{subfigure}[b]{0.6\textwidth}
		\centering
	\input{BEAR.tikz}
	\end{subfigure}
	\hspace*{0cm}
	\begin{subfigure}[b]{0.37\textwidth}
		\centering
	\input{BEAR_stringdiagram.tikz}
	\end{subfigure}
	\caption{A Bayesian network and the string diagram in $\finstoch$ representing it.}\label{fig:ex_bayesian_network}
	\end{figure}
\end{example}

\begin{remark}[On Ordering]\label{rem:ordering}
	The stochastic matrix $f_v$ as in \Cref{def:BN} can be described by a morphism $\tau(v)\to \prod_{w \in\parents{v}} \tau(w)$ in $\finstoch$.
	However, this assignment is only unique up to permutation. It becomes unique once we choose a specific order for $\prod_{w \in\parents{v}} \tau(w)$. Thus for simplicity we work with (totally) \newterm{ordered graphs}. Recall that a DAG $\g$ is ordered if equipped with a \emph{topological ordering}, i.e.\ a total order such that $v\to w$ implies $v <w$. 
\end{remark}
The idea behind the functorial perspective is to cleanly separate `syntax' (the graph) and `semantics' (the probability tables) of a Bayesian network. Viewing the graph as syntax is made possible by the free construction of \Cref{ex:freecdo}. Given an ordered DAG $\g=(V_{\g},E_{\g})$, we define $\cdsyn{\g}$ as the free CD-category given by the signature $(V_\mathcal{G}, \Sigma_{\mathcal{G}})$, where $\Sigma_{\mathcal{G}} \coloneqq \left\lbrace{\input{parents_morph.tikz} \given v \in V_{\mathcal{G}}}\right\rbrace$. As stated in the following proposition, we can identify Bayesian networks based on $\g$ with models of $\g$ in $\finstoch$.

\begin{proposition}[{\cite[Prop.~3.1]{jacobs2019causal_surgery}}]\label{prop:bayesian_functor}
	Let $\g$ be an ordered DAG. 
	Bayesian networks over $\g$ are in bijective correspondence with CD-functors $\cdsyn{\g}\to \finstoch$.
\end{proposition}

We now focus on Markov networks, with the goal of achieving a characterisation analogous to \Cref{prop:bayesian_functor}. 
First, we recall how these models are defined in the literature. 

\begin{definition} 
	Given an undirected graph $\h=(V_{\h}, E_{\h})$, a \newterm{clique} $C$ is a subset of $V_{\h}$ such that for each $v,w\in C$, the set $\lbrace v,w \rbrace$ is an edge. 
The set of cliques is denoted by $\clique{\h}$.
A \newterm{Markov network} over $\h$ is given by an assignment $\tau(v)$ of a finite set for each vertex $v$, together with a function $\phi_C \colon \prod_{v \in C}\tau(v) \to \mathbb{R}^{\ge 0}$, called \emph{factor}, for each clique $C \in \clique{\h}$. 
\end{definition}
Such a Markov network yields a distribution defined by $\omega(V_{\h}) = \frac{1}{Z}\prod_{C\in \clique{\h}} \phi_C(C)$, where $Z$ is a normalisation coefficient. 
Note that $Z$ can be zero, thus $\omega$ is either a probability distribution or identically zero.
When $\omega$ satisfies the latter, we say that the Markov network is \newterm{degenerate}. The unnormalised distribution associated to a Markov network can be represented diagrammatically by comparing all occurrences of the outputs of the factors. Before formulating this construction in whole generality, we show it via an example.

\begin{figure}[!t]
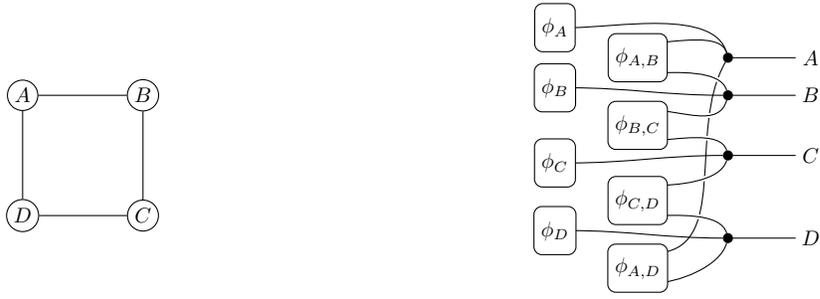

	\centering
	\begin{subfigure}[b]{0.43\textwidth}
		\centering
	\input{undirected_graph.tikz}
	\end{subfigure}
	\hspace*{1cm}
	\begin{subfigure}[b]{0.43\textwidth}
		\centering
	{\input{markov_network_ex.tikz}}
	\end{subfigure}
	\caption{An undirected graph and the string diagram representing its unnormalised distribution.
	Although more graphically complex, the string diagram makes all contributing factors explicit. This approach is also commonly reflected in the theory of PGMs through the use of factor graphs, which are more descriptive.}
	\label{fig:markov_network_ex}
\end{figure}
\begin{example}\label{ex:misconception}
	The undirected graph in \cref{fig:markov_network_ex} illustrates the differences between Markov and Bayesian networks. In this scenario, adapted from \cite[Ex.~3.8]{koller2009probabilistic}, four students --- Alice ($A$), Bob ($B$), Charles ($C$), and Debbie ($D$) --- meet in pairs to work on their class homework. 
	$A$ and $C$ do not get along well, nor do $B$ and $D$, so the only pairs that do \emph{not} meet are these two.
	During class, the professor misspoke, leading to a potential misconception among the students.
	Since $A$ and $C$ do not communicate directly to each other, they can only influence each other through $B$ and $D$. 
	Similarly, $B$ and $D$ only influence each other through $A$ and $C$.
	Note the symmetry of these relationships cannot be captured by a Bayesian network, as it inherently requires a specific choice of directionality among student-vertices. An example of a Markov network over this graph is given by the factors below, where vertices are associated with sets $A=\lbrace a,\no{a}\rbrace$, $B=\lbrace b,\no{b}\rbrace$, $C=\lbrace c,\no{c}\rbrace$, $D=\lbrace d,\no{d}\rbrace$.
	\begin{equation}\label{tab:misconception_factors}
		{\begin{tabular}{|c|c|c|}
		\hline
		\multicolumn{3}{|c|}{$\phi_{AB}$}\\
		\hline\hline
		$a$&$b$& $10$ \\
		\hline 
		$a$ & $\no{b}$ & $1$\\
		\hline 
		$\no{a}$& $b$ & $5$\\
		\hline
		$\no{a}$&$\no{b}$ & $30$\\
		\hline
		\end{tabular}}	
		\hspace{0.3cm}
		{\begin{tabular}{|c|c|c|}
		\hline
		\multicolumn{3}{|c|}{$\phi_{BC}$}\\
		\hline\hline
		$b$&$c$& $100$ \\
		\hline 
		$b$ & $\no{c}$ & $1$\\
		\hline 
		$\no{b}$& $c$ & $1$\\
		\hline
		$\no{b}$&$\no{c}$ & $100$\\
		\hline
		\end{tabular}}	
		\hspace{0.3cm}
		{\begin{tabular}{|c|c|c|}
		\hline
		\multicolumn{3}{|c|}{$\phi_{CD}$}\\
		\hline\hline
		$c$&$d$& $1$ \\
		\hline 
		$c$ & $\no{d}$ & $100$\\
		\hline 
		$\no{c}$& $d$ & $100$\\
		\hline
		$\no{c}$&$\no{d}$ & $1$\\
		\hline
		\end{tabular}}	
		\hspace{0.3cm}
		{\begin{tabular}{|c|c|c|}
		\hline
		\multicolumn{3}{|c|}{$\phi_{AD}$}\\
		\hline\hline
		$a$&$d$& $100$ \\
		\hline 
		$a$ & $\no{d}$ & $1$\\
		\hline 
		$\no{a}$& $d$ & $1$\\
		\hline
		$\no{a}$&$\no{d}$ & $100$\\
		\hline
		\end{tabular}}
	\end{equation} 

	In our interpretation, $\no{x}$ indicates that $X$ does not have the misconception. As we do not wish to consider an inherent difference between the various students, all the factors over a single node are omitted (we can simply set them to be constantly one). 
	The chosen numbers highlight some key aspects of the students relationships: $C$ and $D$ are prone to disagree, while all other pairs tend to agree. Additionally, $A$ and $B$ are less likely to have the misconception, and the fact that $\phi_{AB}(a,\no{b})< \phi_{AB} (\no{a},b)$ indicates that when they do disagree, it is more plausible that $B$ had the misconception.
\end{example}

To establish an analogue of \cref{prop:bayesian_functor} for Markov networks, we use the free construction of \cref{ex:freecor}.
Given any \emph{ordered} undirected graph $\h=(V_\mathcal{H}, E_{\mathcal{H}})$, consider the signature given by $(V_\mathcal{H}, \Sigma_{\mathcal{H}})$, with $\Sigma_{\mathcal{H}} \coloneqq \left\lbrace{\begin{tikzpicture}
	\begin{pgfonlayer}{nodelayer}
		\node [style=small box] (0) at (0, 0) {};
		\node [style=none] (1) at (1.5, 0) {};
		\node [style=none] (2) at (2, 0) {$C$};
	\end{pgfonlayer}
	\begin{pgfonlayer}{edgelayer}
		\draw (1.center) to (0);
	\end{pgfonlayer}
\end{tikzpicture}
 \given C \in \clique{\mathcal{H}}}\right\rbrace$, and define $\syn{\mathcal{H}} \coloneqq \freehyp{V_\mathcal{H}, \Sigma_{\mathcal{H}}}$.
Note that the order here is necessary to give a well-defined box for each clique, since otherwise we would not know how to order the outputs (cf.~\cref{rem:ordering}).
\begin{proposition}\label{prop:markov_functor}
	Let $\h$ be an ordered undirected graph.
	Markov networks over $\h$ are in bijective correspondence to hypergraph functors $\syn{\h} \to \mat$.
\end{proposition}

For instance, the Markov network of \Cref{ex:misconception} yields $V_\mathcal{H} = \{A,B,C,D\}$ and $\Sigma_{\mathcal{H}}$ including all the generators $\phi$ appearing in the string diagram of \cref{fig:markov_network_ex}, and the functor $\syn{\h}\to \mat$ is defined according to the factors defined in \eqref{tab:misconception_factors}.

\begin{proof}
	We note that hypergraph functors (resp.\ CD-functors) from free hypergraph categories (resp.\ free CD-categories) can be freely defined by describing where the generating objects and morphisms are sent. 
	
	Whenever we have a Markov network, we can define the hypergraph functor on generators by setting $
	F(v)\coloneqq \tau(v)$ and $F\left(\begin{tikzpicture}
	\begin{pgfonlayer}{nodelayer}
		\node [style=small box] (0) at (0, 0) {};
		\node [style=none] (1) at (1.5, 0) {};
		\node [style=none] (2) at (2, 0) {$C$};
	\end{pgfonlayer}
	\begin{pgfonlayer}{edgelayer}
		\draw (1.center) to (0);
	\end{pgfonlayer}
\end{tikzpicture}
\right) \coloneqq \phi_C$.
	Conversely, given a hypergraph functor $F \colon \syn{\mathcal{H}} \to \mat$, one simply uses the definitions above in the converse direction: $
	\tau(v) \coloneqq F(v)$ and $\phi_C \coloneqq F\left(\begin{tikzpicture}
	\begin{pgfonlayer}{nodelayer}
		\node [style=small box] (0) at (0, 0) {};
		\node [style=none] (1) at (1.5, 0) {};
		\node [style=none] (2) at (2, 0) {$C$};
	\end{pgfonlayer}
	\begin{pgfonlayer}{edgelayer}
		\draw (1.center) to (0);
	\end{pgfonlayer}
\end{tikzpicture}
\right)$.
	As the two functions defined are clearly the inverse of one another, the statement follows.
\end{proof}

\section{Syntax Functoriality}\label{sec:syntax_functoriality}

A useful result for our developments, not appearing in~\cite{jacobs2019causal_surgery}, is that $\cdsyn{\g}$ itself may be viewed as a functorial construction. Its codomain is the category $\cdcat$ of CD-categories from \Cref{def:cdcat}, while its domain is defined precisely in the following.
\begin{definition}\label{def:odag}
	The category $\odag$ of ordered DAGs is defined as follows:
	\begin{itemize}
		\item Objects are ordered DAGs, i.e.\ DAGs equipped with a total order such that $v \to w$ implies $v<w$;
		\item Morphisms are order-preserving graph homomorphisms, i.e.\ $\alpha \colon \g_1 \to \g_2$ is a function on the sets of vertices $V_{\g_1}\to V_{\g_2}$ that preserves the order and the edges.\footnote{If two vertices $x$ and $y$ have the same image under $\alpha$, we assume that the possible edges $(x,y)$ and $(y,x)$ are respected.}
	\end{itemize}
\end{definition}
The claimed functoriality requires a rigorous treatment of a special type of composition that allows a given variable to be considered multiple times.
For better understanding, let us start with a defining example.

Note $\bullet$, the \newterm{one-vertex graph}, is the final object in $\odag$. 
Thus given an ordered DAG $\g$, there always exists a unique morphism $\g \to \bullet$. 
Another important observation is that, by definition, $\cdsyn{\bullet}$ is generated by the monoidal signature consisting of a single generating object $\bullet$ and a single generating morphism \begin{tikzpicture}
	\begin{pgfonlayer}{nodelayer}
		\node [style=small box] (0) at (0, 0) {};
		\node [style=none] (1) at (2, 0) {};
	\end{pgfonlayer}
	\begin{pgfonlayer}{edgelayer}
		\draw (1.center) to (0);
	\end{pgfonlayer}
\end{tikzpicture}
, of type $I \to \bullet$. 
The idea is then to use a CD-functor $\cdsyn{\bullet}\to \cdsyn{\g}$ to describe the string diagram associated to the probability distribution (for \cref{ex:BEAR}, this is displayed in \cref{fig:ex_bayesian_network}).
We gain two important insights. First, we expect the functor to be \emph{contravariant}, since $\g\to \bullet$ corresponds to $\cdsyn{\bullet}\to \cdsyn{\g}$. Second, we must clarify what it means to compose by copying the information (for \cref{ex:BEAR}, the object $E$ must be copied to act as a parent of both $A$ and $R$).
Indeed, formalising copy-composition --- and later, compare-composition --- is the central focus of the present section.

Let us start by introducing some necessary notation.

\begin{notation}
	For any morphism $\phi$ in a CD-category, let us denote by $\In{\phi}$ and $\Out{\phi}$ the sets of inputs and outputs.
	Similarly, for any set of morphisms $S$, we consider $\Out{S}\coloneqq \bigcup_{\phi \in S} \Out{\phi}$, whereas $\In{S} \coloneqq \bigcup_{\phi \in S} \In{\phi}\cap \Out{S}^c$.
\end{notation}

The asymmetry between the notations $\In{S}$ and $\Out{S}$ is motivated by their use in the following.
\begin{definition}\label{def:copycomp}
	Let $\g$ be an ordered DAG. 	
	Consider a set of generators $S\subseteq\Sigma_{\g}$ and let $\phi$ be the generator in $S$ whose output is the biggest element of $\Out{S}$.
	By induction, we define $\copycomp{\emptyset}\coloneqq \id_I$ and 
	\[
	{\input{copycomp_S.tikz}}\qquad \coloneqq \qquad 
	{\input{copycomp_def.tikz}}
	\] 
	where $S'\coloneqq S \setminus \lbrace \phi \rbrace$, and $\pi\colon \In{S'} \cup \Out{S'}\to (\In{S'} \cup \Out{S'})\cap \In{\phi}$
	is the marginalization (i.e., it deletes all the occurring vertices that are not inputs of $\phi$).
	In the string diagram we omitted the permutations of the inputs to avoid using additional notation.
	For a given $S$, we refer to $\copycomp{S}$ as the \newterm{copy-composition} of $S$.
\end{definition}
By definition, if $S=\lbrace \phi \rbrace$, then $\copycomp{S}=\phi$. 

\begin{remark}\label{rem:smithe}
	Although similar, the copy-composition above differs from the one introduced by Smithe in \cite{smithe2024copycomposition}.
	Indeed, for our purposes, copy-composition must allow inputs and outputs to simply overlap without entirely matching, which extends beyond Smithe's original notion. 
	On the other hand, this relaxed notion requires additional care as everything depends on a specific order. 
	This is why we have stated it only for sets of generators $S \subseteq \Sigma_{\g}$,  making it specific to the syntax categories $\cdsyn{\g}$ and $\syn{\g}\coloneqq \freehyp{V_{\g},\Sigma_{\g}}$, as used abundantly from \cref{sec:moralisation} onwards.
\end{remark}

\begin{remark}[Semantics with the Copy-Composition]\label{rem:semantics_copycomp}
	Consider the DAG $\begin{tikzpicture}
	\begin{pgfonlayer}{nodelayer}
		\node [style=bw] (15) at (-1.5, 0) {$A$};
		\node [style=bw] (31) at (1.5, 0) {$B$};
	\end{pgfonlayer}
	\begin{pgfonlayer}{edgelayer}
		\draw [style=arrow] (15) to (31);
	\end{pgfonlayer}
\end{tikzpicture}
$, and a CD-functor $F\colon \cdsyn{\g}\to \finstoch$. Set $\omega\coloneqq F(I\to A)$ and $f\coloneqq F(A \to B)$. 
	Then $F(\copycomp{\Sigma_{\g}}) = \input{fcopyw.tikz}$, which is the probability distribution given by $A \tensor B \owns (a,b)\mapsto f(b\given a) \omega(a)$, where the occurrences of $a$ as output of $\omega$ and as input of $f$ are identified. 
	In other words, $\input{fcopyw.tikz} = f(B\given A)\omega(A)$.

	In general, the interpretation of copy-composition, particularly when given a CD-functor $\cdsyn{\g}\to \finstoch$, is to identify all occurrences of the same variable. 
\end{remark}

\begin{lemma}\label{lem:copycomp}
	Let $S$ and $T$ be two sets of generators such that for each $\phi\in S$ and $\psi\in T$, $\Out{\phi} <\Out{\psi}$. Then 
	\[
		{\input{copycomp_ST.tikz}} \qquad = \qquad {\input{copycomp_SandT.tikz}}
	\]
	where $\pi$ marginalizes all the inputs and outputs of $S$ that are not inputs of $T$.
\end{lemma}
\begin{proof}
	By induction, we can assume that statement holds for $T'\coloneqq T \setminus \lbrace \phi\rbrace$, where $\phi$ is the generator whose output is the biggest element. Then 
	
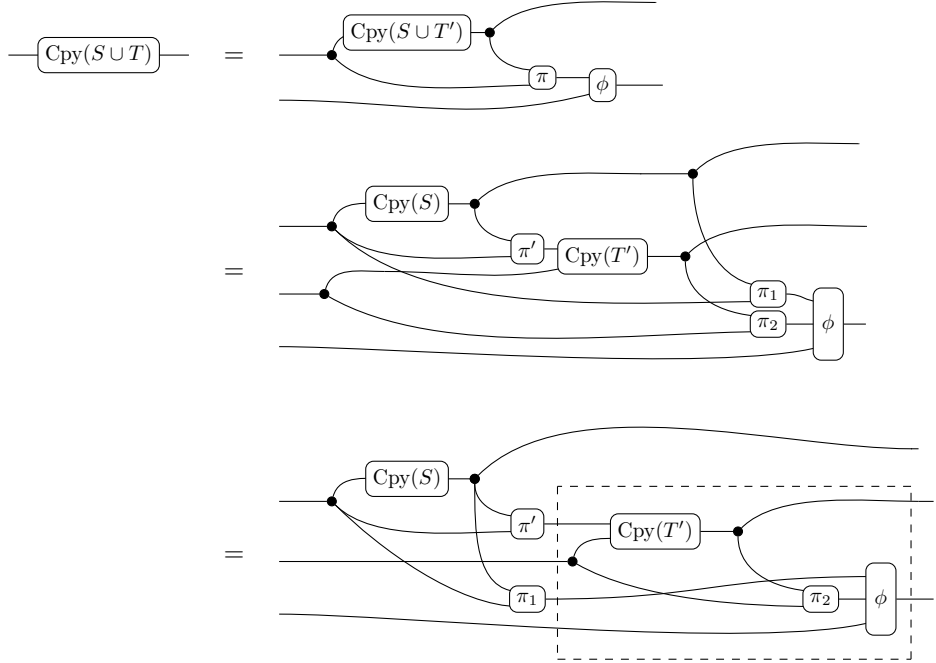
\[
	\begin{aligned}
		{\input{copycomp_ST.tikz}}\quad &=\quad {\input{copycomp_ST1_phi.tikz}}\\
		&=\quad {\input{copycomp_S_T1_phi.tikz}}\\
		&=\quad {\input{copycomp_S_T.tikz}}
	\end{aligned}
	\]
	where $\pi=\pi_1 \otimes \pi_2$, $\pi_1$ only considers the inputs obtained from $S$ and $\pi_2$ those obtained from $T'$. Since the dashed box corresponds to $\copycomp{T}$, the statement is shown.
\end{proof}

\begin{notation}
	Let $\alpha\colon \g \to \g'$ be an order-preserving graph homomorphism, and let $v \in V_{\g'}$.
	We write $S_v^{\alpha}$, or $S_v$ when there is no confusion, for the set of generators $\phi \in \Sigma_{\g}$ such that $\Out{\phi} \subseteq \alpha^{-1}(v)$.
\end{notation}
\begin{remark}\label{rem:sv_properties}
	Let us highlight some important properties of $S_v$.
	\begin{enumerate}
		\item\label{it:sv_out} As every vertex in $V_{\g}$ is an output for some generator in $\Sigma_{\g}$, $\Out{S_v}=\alpha^{-1}(v)$.
		\item\label{it:sv_partition} The collection $\lbrace S_v \mid v\in \alpha(\g) \rbrace$ gives a partition of $\Sigma_{\g}$ because the same holds for the preimage.
		\item\label{it:sv_in_all} We have 
		\[
			\bigcup_{\phi \in S_v} \In{\phi} = \bigcup_{\phi \in S_v} \parents{\Out{\phi}} = \bigcup_{w \in \alpha^{-1}(v)} \parents{w} = \parents{\alpha^{-1}(v)},
		\]
		and therefore $\In{S_v}=\parents{\alpha^{-1}(v)}\cap \alpha^{-1}(v)^c$ by \cref{it:sv_out}. In particular, from this equality we infer that $\In{S_v}\subseteq \alpha^{-1}(\parents{v})$.
	\end{enumerate}
\end{remark}

We are now ready to define the association on morphisms. 
We will tacitly use \cref{rem:sv_properties} in the definition.

\begin{definition}\label{def:cdsyn_morph}
	Let $\alpha \colon \g \to \g'$ be an order-preserving graph homomorphism.
	We define a functor $\cdsyn{\alpha}\colon \cdsyn{\g'}\to \cdsyn{\g}$ as follows.
	\begin{itemize}
		\item On objects, $v$ is sent to the tensor product given by $\alpha^{-1}(v)$ (and ordered according to the order on $\g$).
		\item On morphisms, $\cdsyn{\alpha}$ is given by 
		\[
			{\input{parents_morph.tikz}} \quad \longmapsto \quad 
			{\input{synf_dags.tikz}}
		\]
		where $\In{S_v}^c\coloneqq \alpha^{-1}(\parents{v})\setminus \In{S_v}$. 
		As in the previous definition, we omit some permutation of objects to avoid additional notation.
	\end{itemize}
\end{definition}
Two concrete examples of $\cdsyn{\alpha}$ are provided by \cref{fig:graphhoms_bn_ex}.
 
\begin{figure}[tp]
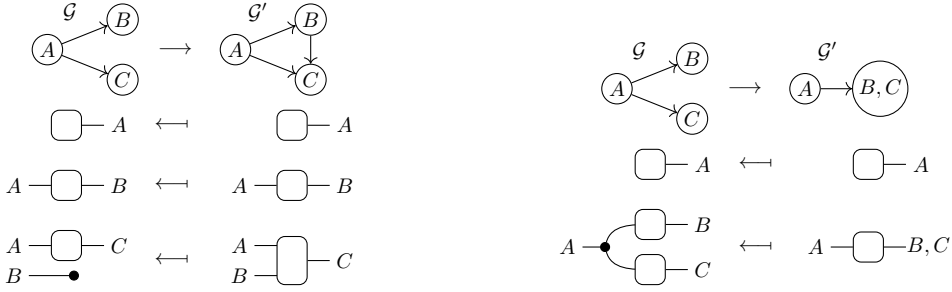

	\begin{subfigure}[b]{0.44\textwidth}
		\centering
		\input{odag_ex1.tikz}
		\par\vspace{1ex}
		{\input{odag_ex1_stringdiagram.tikz}}
	\subcaption{The missing connection $B\to C$ in $\g$ results in a deletion of the $B$ input.}\label{subfig:graphhoms_bn_ex1}
	\end{subfigure}
	\hspace*{0.6cm}
	\begin{subfigure}[b]{0.44\textwidth}
		\centering
		\input{odag_ex2.tikz}
		\par\vspace{1ex}
		{\input{odag_ex2_stringdiagram.tikz}}
	\subcaption{The added distinction of $B$ and $C$ gives rise to a copy.}\label{subfig:graphhoms_bn_ex2}
	\end{subfigure}
	\caption{Examples of the contravariant action of the functor of \Cref{thm:graphhom-cd} on morphisms $\g \to \g'$ (top), resulting in CD-functors (bottom), of which we describe the action on generators of $\cdsyn{\g}$.}\label{fig:graphhoms_bn_ex}
\end{figure}

\begin{lemma}\label{lem:copycomp2}
	Let $\alpha \colon \g \to \g'$ be an order-preserving graph homomorphism and let $T\subset\Sigma_{\g'}$. Then 
	\[
	\cdsyn{\alpha}\left(
	{\input{copycompT.tikz}}\right) \qquad =\qquad  
	{\input{synf_dags_lem.tikz}}
	\]
\end{lemma}
\begin{proof}
	By induction, let us assume the statement is true for $T'= T \setminus\lbrace \phi\rbrace$, where $\phi$ is the generator whose output is the biggest element. Then
	\[
	\begin{aligned}
		\cdsyn{\alpha}\left({\input{copycompT.tikz}}\right) \quad &=\quad \cdsyn{\alpha}\left({\input{copycomp_T1_phi.tikz}}\right)\\
		&=\quad {\input{copycomp_T1phi_im.tikz}}
	\end{aligned}
	\] 
	and we conclude by \cref{lem:copycomp} and \cref{it:sv_partition} of \cref{rem:sv_properties}.
\end{proof}

\begin{theorem}\label{thm:graphhom-cd}
	The association 
	\[
	\begin{array}{rccc}
		\cdsyn{} \colon& \odag & \to & \cdcat\\
		& \g,\quad \alpha \colon \g \to \g' &\mapsto & \cdsyn{\g},\quad \cdsyn{\alpha}\colon \cdsyn{\g'}\to \cdsyn{\g}
	\end{array}
	\]
	is a contravariant functor.
\end{theorem}

\begin{proof}
	We note that indeed $\cdsyn{\id}=\id$, so we are left to prove composition.
	Let $\alpha\colon \g_1 \to \g_2$ and $\beta\colon \g_2\to \g_3$. 
	On objects, $\cdsyn{\alpha}\cdsyn{\beta}(v) = \cdsyn{\beta\alpha} (v)$ holds by properties of the preimage and the fact that $\beta$ and $\alpha$ are order-preserving.
	
	Let us now consider a generator $\input{parents_morph.tikz}$ in $\cdsyn{\g_3}$. 
	By \cref{lem:copycomp2}, it suffices to show that $\bigcup_{\phi \in S_v^{\beta}} {S_{\Out{\phi}}^{\alpha}} = {S_v^{\beta\alpha}}$. But this follows immediately from \cref{rem:sv_properties}:
	\[
		\bigcup_{\phi \in S_v^{\beta}} S_{\Out{\phi}}^{\alpha} = \bigcup_{w \in \beta^{-1}(v)} S_w^{\alpha} = \lbrace \phi \mid \Out{\phi} \subseteq \alpha^{-1}(w) \text{ for some }w \in \beta^{-1}(v) \rbrace = {S_v^{\beta\alpha}}.
	\]
\end{proof}

\begin{remark}[Factorising Distributions]\label{rem:factorising_distributions_DAG}
	We now return to our defining example. 
	Via the functor $\cdsyn{}$, the contraction $\g \to \bullet$ gives the CD-functor $!_{\g}\colon \cdsyn{\bullet}\to \cdsyn{\g}$, which sends the single generator $\begin{tikzpicture}
	\begin{pgfonlayer}{nodelayer}
		\node [style=small box] (0) at (0, 0) {};
		\node [style=none] (1) at (2, 0) {};
	\end{pgfonlayer}
	\begin{pgfonlayer}{edgelayer}
		\draw (1.center) to (0);
	\end{pgfonlayer}
\end{tikzpicture}
$ to $\copycomp{\Sigma_{\g}}$.

	When considering a CD-functor $F\colon \cdsyn{\g}\to \finstoch$, the distribution $\omega$ associated to $\cdsyn{\bullet}\to \cdsyn{\g}\to \finstoch$ is then defined by $F(\copycomp{\Sigma_{\g}})$.
	Since CD-functors preserve $\input{copy.tikz}$ and $\begin{tikzpicture}
	\begin{pgfonlayer}{nodelayer}
		\node [style=bn] (35) at (0.75, 0) {};
		\node [style=none] (36) at (-0.75, 0) {};
	\end{pgfonlayer}
	\begin{pgfonlayer}{edgelayer}
		\draw (35) to (36.center);
	\end{pgfonlayer}
\end{tikzpicture}
$, we have a factorisation of $\omega$. 
	More explicitly, set $f_v\coloneqq F(\input{parents_morph.tikz})$.
	Then, $\omega(V_{\g})= \prod_{v} f_v(v\given \parents{v})$, where we identified all occurrences of the same variable, as discussed in \cref{rem:semantics_copycomp}.
\end{remark}


Just as in the case of Bayesian networks, also the construction of $\syn{\h}$ yields a functor. Its codomain is $\hypcat$, introduced in \Cref{def:hypcat}, while its domain is the following category.

\begin{definition}
	The category $\ougr$ of ordered undirected graphs is defined as follows:
	\begin{itemize}
		\item Objects are ordered undirected graphs;
		\item Morphisms are order-preserving graph homomorphisms, i.e.\ $\alpha \colon \h_1 \to \h_2$ is a function on the sets of vertices $V_{\h_1}\to V_{\h_2}$ that preserves the order and the edges.\footnote{If two vertices $x$ and $y$ have the same image under $f$, we assume that the possible edge $\lbrace x,y\rbrace$ is respected.}
	\end{itemize}
\end{definition}

Similarly to the case of $\cdsyn{}$, we need to define what it means to consider several morphisms together. 
Importantly, as here we do not need to take care of directionality, this definition can be extended to a general situation, which will be important from \cref{sec:moralisation} onwards (see \cref{lem:graphcopycomp}).
\begin{definition}
	A morphism $\phi \colon I\to C$ in a hypergraph category is a \newterm{factor} if it comes equipped with a sequence of objects $(X_1 ,\dots , X_n)$ such that $X_1 \tensor \cdots \tensor X_n=C$. 
	
	We write $\Out{\phi}\coloneqq \lbrace X_1, \dots, X_n\rbrace$, and, more generally, $\Out{S}\coloneqq \bigcup_{\phi \in S} \Out{\phi}$ for a finite set of factors $S$. 
\end{definition}

A finite set of factors $S$ is assumed to come with an ordering of $\Out{S}$ to avoid the permutation issue discussed in \cref{rem:ordering}.

In the categories $\syn{\g}$ and $\syn{\h}$, every morphism $I\to C$ is tacitly assumed to become a factor by considering the vertices of the associated graph. In particular, every set of factors in these categories comes with an induced ordering.

\begin{definition}
	For every finite set of factors $S$ in a hypergraph category, we define by induction $\compcomp{\emptyset}=\id_I$ and
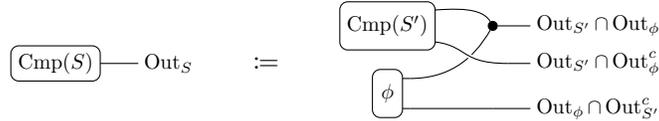
	\[
		{\begin{tikzpicture}
	\begin{pgfonlayer}{nodelayer}
		\node [style=morphism] (79) at (-5.75, 0) {$\compcomp{S}$};
		\node [style=none] (80) at (-3, 0) {};
		\node [style=none] (81) at (-5.25, 0) {};
		\node [style=none] (83) at (-2, 0) {$\Out{S}$};
	\end{pgfonlayer}
	\begin{pgfonlayer}{edgelayer}
		\draw (81.center) to (80.center);
	\end{pgfonlayer}
\end{tikzpicture}
} \qquad \coloneqq \qquad {\input{compare_comp_def.tikz}}
	\]
	where $\phi$ is arbitrarily chosen and $S'\coloneqq S \setminus \lbrace \phi \rbrace$. Permutations of the outputs on the right hand side are omitted for brevity. 
	For a given $S$, we refer to $\compcomp{S}$ as the \newterm{compare-composition} of $S$.
\end{definition}
\begin{lemma}\label{lem:compcomp}
	The compare-composition is well-defined, i.e. it does not depend on the choice of $\phi$. 
	Moreover, for any disjoint sets of factors $S$ and $T$,
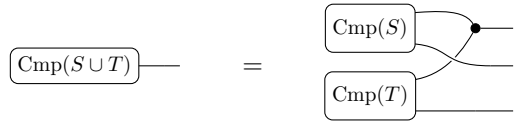
	\[
		{\begin{tikzpicture}
	\begin{pgfonlayer}{nodelayer}
		\node [style=morphism] (79) at (0, 0) {$\compcomp{S\cup T}$};
		\node [style=none] (80) at (3, 0) {};
		\node [style=none] (81) at (0.5, 0) {};
		\node [style=none] (82) at (3.5, 0) {};
	\end{pgfonlayer}
	\begin{pgfonlayer}{edgelayer}
		\draw (81.center) to (80.center);
		\draw (80.center) to (82.center);
	\end{pgfonlayer}
\end{tikzpicture}
}\qquad =\qquad {\input{compare_comp2.tikz}}	
	\]
\end{lemma}
The proof is omitted as it immediately follows from associativity and commutativity of \input{compare.tikz}.

\begin{remark}[Semantics with the Compare-Composition]\label{rem:semantics_compcomp}
	\cref{rem:semantics_copycomp} can be translated to this setting, meaning that the compare-composition identifies all occurrences of the same variable on the different factors. 
	For example, take two factors $\phi\colon I \to A \tensor B$ and $\psi\colon I \to A \tensor C$ in $\mat$. 
	Then $\compcomp{\lbrace \phi,\psi\rbrace}$ is given by the factor $A\tensor B\tensor C\owns (a,b,c)\mapsto \phi(a,b)\psi(a,c)$, also written as $\phi(AB)\psi(AC)$.
\end{remark}

\begin{definition}\label{def:syn_morph}
	Let $\alpha \colon \h \to \h'$ be an order-preserving graph homomorphism. The hypergraph functor $\syn{\alpha}\colon \syn{\h'} \to \syn{\h}$ is defined as follows:
	\begin{itemize}
	\item It maps an object $v$ to the tensor product given by $\alpha^{-1}(v)$.
	\item On morphisms, 
	\[
		{\begin{tikzpicture}
	\begin{pgfonlayer}{nodelayer}
		\node [style=small box] (0) at (0, 0) {};
		\node [style=none] (1) at (1.5, 0) {};
		\node [style=none] (2) at (2, 0) {$C$};
	\end{pgfonlayer}
	\begin{pgfonlayer}{edgelayer}
		\draw (1.center) to (0);
	\end{pgfonlayer}
\end{tikzpicture}
}\qquad \longmapsto \qquad {\input{synf_ugr.tikz}}
	\]
	where $S_C\coloneqq \lbrace\phi \colon I \to D \mid \alpha(D)=C\rbrace$.
	\end{itemize}
\end{definition}
Two concrete examples of $\syn{\alpha}$ are provided by \cref{fig:graphhoms_mn_ex}. 	

\begin{figure}[tp]
	\begin{subfigure}[b]{0.44\textwidth}
		\centering
		\input{ougr_ex1.tikz}
		\par\vspace{1ex}
		{\input{ougr_ex1_stringdiagram.tikz}}
	\subcaption{The missing connection $B\edge C$ in $\h$ results in disregarding the clique $\lbrace B,C\rbrace$ and consequently also $\lbrace A,B,C\rbrace$. For brevity, we used $X$ as a placeholder for the other cliques $\lbrace A\rbrace$, $\lbrace B\rbrace$, $\lbrace C\rbrace$, $\lbrace A,B\rbrace$, and $\lbrace A,C\rbrace$.
	}\label{subfig:graphhoms_mn_ex1}
	\end{subfigure}
	\hspace*{0.6cm}
	\begin{subfigure}[b]{0.44\textwidth}
		\centering
		\input{ougr_ex2.tikz}
		\par\vspace{1ex}
		{\input{ougr_ex2_stringdiagram.tikz}}
	\subcaption{The added distinction of $B$ and $C$ gives rise to a compare.}\label{subfig:graphhoms_mn_ex2}
	\end{subfigure}
\caption{Examples of the contravariant action of the functor of \Cref{thm:graphhom-hyp} on morphisms $\h \to \h'$ (top), resulting in hypergraph functors (bottom), of which we describe the action on generators of $\syn{\h'}$.}\label{fig:graphhoms_mn_ex}
\end{figure}

\begin{lemma}\label{lem:compcomp2}
	Let $\alpha \colon \h \to \h'$ be an order-preserving graph homomorphism and let $T\subseteq \Sigma_{\h'}$. Then 
	\[
	\syn{\alpha}\left({\begin{tikzpicture}
	\begin{pgfonlayer}{nodelayer}
		\node [style=morphism] (58) at (0, 0) {$\compcomp{T}$};
		\node [style=none] (61) at (0, 0) {};
		\node [style=none] (62) at (2.75, 0) {};
	\end{pgfonlayer}
	\begin{pgfonlayer}{edgelayer}
		\draw (61.center) to (62.center);
	\end{pgfonlayer}
\end{tikzpicture}
}\right) \qquad =\qquad  {\input{synf_ugr_lem.tikz}}
	\]
\end{lemma}
This is proven by induction using \cref{lem:compcomp}, following the approach of the proof of \cref{lem:copycomp2}.

\begin{theorem}\label{thm:graphhom-hyp}
	The association 
	\[
	\begin{array}{rccc}
		\syn{} \colon& \ougr & \to & \hypcat\\
		& \h,\quad \alpha \colon \h \to \h' &\mapsto & \syn{\h},\quad \syn{\alpha}\colon \syn{\h'}\to \syn{\h}
	\end{array}
	\]
	is a contravariant functor.
\end{theorem}
\begin{proof}
	A direct check shows that $\syn{\id}=\id$, so we focus on composition. Let $\alpha\colon \h_1 \to \h_2$ and $\beta\colon \h_2 \to \h_3$.
	As in the proof of \cref{thm:graphhom-cd}, $\syn{\alpha}\comp \syn{\beta}(v)=\syn{\beta\alpha}(v)$ because the considered graph homomorphisms are order-preserving.

	Regarding composition, let $\phi \colon I \to C$ in $\syn{\h_3}$. By \cref{lem:compcomp2}, it suffices to show that $\bigcup_{\phi \in S_C^{\beta}} S_{\Out{\phi}}^{\alpha} = S_C^{\beta\alpha}$, where $S_Y^{\ell}\coloneqq \lbrace \psi \colon I \to X \mid \ell(X)=Y \rbrace$. 
	By unpacking the definition,
	\[
	\bigcup_{\phi \in S_C^{\beta}} S_{\Out{\phi}}^{\alpha} = \lbrace \psi \colon I \to D \,\given\,  \exists\,  E \in \clique{\h_2}\text{ such that } \alpha(D)=E\text{ and } \beta(E)=C \rbrace.
	\] 
	Since graph homomorphisms send cliques to cliques, whenever $D$ is a clique, $\alpha(D)$ is a clique as well. 
	Therefore, the condition on the right can be shortened to $\beta\alpha(D)=C$, so the wanted equality $\bigcup_{\phi \in S_C^{\beta}} S_{\Out{\phi}}^{\alpha} = S_C^{\beta\alpha}$ indeed holds.
\end{proof}

\begin{remark}[Factorising Distributions]\label{rem:factorising_distributions_UGr}
	Similarly to \cref{rem:factorising_distributions_DAG}, we consider the contraction map $\h \to \bullet$, which via $\syn{}$ yields the hypergraph functor $!_{\h}\colon \syn{\bullet}\to \syn{\h}$ sending the single generator $\begin{tikzpicture}
	\begin{pgfonlayer}{nodelayer}
		\node [style=small box] (0) at (0, 0) {};
		\node [style=none] (1) at (2, 0) {};
	\end{pgfonlayer}
	\begin{pgfonlayer}{edgelayer}
		\draw (1.center) to (0);
	\end{pgfonlayer}
\end{tikzpicture}
$ to $\compcomp{\Sigma_{\h}}$.
	For example, considering \cref{ex:misconception}, $\syn{\bullet}\to \syn{\h}$ describes the string diagram in \cref{fig:markov_network_ex}.

	We now consider a hypergraph functor $\Phi\colon \syn{\h}\to \mat$ and set $\phi_C\coloneqq \Phi(\begin{tikzpicture}
	\begin{pgfonlayer}{nodelayer}
		\node [style=small box] (0) at (0, 0) {};
		\node [style=none] (1) at (1.5, 0) {};
		\node [style=none] (2) at (2, 0) {$C$};
	\end{pgfonlayer}
	\begin{pgfonlayer}{edgelayer}
		\draw (1.center) to (0);
	\end{pgfonlayer}
\end{tikzpicture}
)$.
	Then, the distribution $\omega$ associated to the composition $\syn{\bullet}\to \syn{\h}\to \mat$ is defined by $\Phi(\compcomp{\Sigma_{\h}})$, and since hypergraph functors preserve $\input{compare.tikz}$, we have a factorisation $\omega(V_{\h})= \prod_{C} \phi_C(C)$, where we identified all occurrences of the same variable, as discussed in \cref{rem:semantics_compcomp}.
\end{remark}



\section{Irredundant Networks via Functorial Factorisation}\label{sec:irredundant_networks}

When the probability distribution associated to a Bayesian network does not have full support, i.e.\ there is some value $x$ such that $\omega(x)=0$, the Bayesian network structure has an inherent redundancy.
For example, consider the family of Bayesian networks $F_p$, one for each $p \in [0,1]$,  described by 
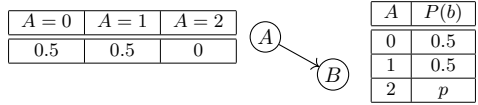
\[
\input{dag_AB.tikz}
\]
where $A=\lbrace 0,1,2\rbrace$, $B=\lbrace b,\no{b} \rbrace$.
All of these networks are associated to the same probability distribution $\omega(A,B)$, but they are different according to Definition~\ref{def:BN} (and \Cref{prop:bayesian_functor}). However, distinguishing them is not always desirable, for instance when studying the conditional independencies of $\omega(A,B)$. The important property is the existence, rather than the uniqueness of a Bayesian network representing the independencies of $\omega(A,B)$. 
	For our developments, it is important to account for this different perspective, which is in fact more natural for the present purposes.
		
\begin{definition}[Irredundant Bayesian Network] \label{def:irredundantBN}
An \newterm{irredundant} Bayesian network over $\g$ is a probability distribution $\omega$, together with an assignment $\tau(v)$ of a finite set for any vertex $v \in V_{\g}$, such that $\omega(V_{\g})= \prod_{v \in V_{\g}} f_v (v \given \parents{v})$ for some stochastic matrices $f_v \colon \tau(v)\times \prod_{w \in\parents{v}} \tau(w)\to [0,1]$.
\end{definition}

\begin{remark}\label{rem:irredundant_distinction}
	In our previous work \cite{lorenzin2025moralisation}, we defined irredundant Bayesian networks in a slightly different manner: the assignment that is fixed once and for all in \cref{def:irredundantBN} was instead allowed to vary. 
	While allowing the assignment to vary simplifies the treatment, it introduces several issues. 
	First, it deviates from the standard theoretical framework, where the assignment determines how each vertex corresponds to a random variable—information that is typically part of the initial data, not something inferred from the network structure. 
	Second, this flexibility can lead to a more superficial or trivial analysis. 
	For these reasons, we here adopt a formulation that aligns more closely with standard theory and provides a more rigorous foundation.
\end{remark}

This irredundant version is more common in the study of conditional independencies. For example, recall that the d-separation criterion ensures when a certain probability distribution admits the factorisation above by means of (global or local) Markov properties. 
In particular, with this definition we can also connect ordered graph homomorphisms and I-maps (see \cref{prop:Imap_DAG} below).

We now introduce a characterisation analogous to~\Cref{prop:bayesian_functor} for the irredundant case. 
Recall that the one-vertex graph $\bullet$ is associated to $\cdsyn{\bullet}$, the free CD-category generated by one object, also denoted by $\bullet$, and one morphism of type $I\to \bullet$.
Therefore, any CD-functor $\cdsyn{\bullet}\to \finstoch$ is entirely captured by setting $\bullet \mapsto X$ and $\begin{tikzpicture}
	\begin{pgfonlayer}{nodelayer}
		\node [style=small box] (0) at (0, 0) {};
		\node [style=none] (1) at (2, 0) {};
	\end{pgfonlayer}
	\begin{pgfonlayer}{edgelayer}
		\draw (1.center) to (0);
	\end{pgfonlayer}
\end{tikzpicture}
 \mapsto \omega$, where $\omega \colon I \to X$ is a probability distribution on $X$. 
Additionally, we need to capture the assignment of vertices without requiring a precise choice of the stochastic matrices. To this end, we consider the following.

\begin{definition}\label{def:vcat}
	Let $\g$ be an ordered DAG. We define the \newterm{vertex-list category} $\vcat{\g}$ as the (coloured) PROP given by the monoidal signature $(V_{\g}, \emptyset)$. More explicitly, objects are finite lists of vertices of $\g$, while the only morphisms are permutations and identities.
\end{definition}

Note that there is an obvious symmetric monoidal functor $\vcat{\g}\hookrightarrow \cdsyn{\g}$ defined by sending each vertex to itself. Moreover, an association $\tau$ of a set for each vertex is exactly the data of a symmetric monoidal functor $\vcat{\g} \to \finstoch$.
Further, this construction is functorial.

\begin{proposition}\label{prop:vcat_functor}
	The association 
		\[
	\begin{array}{rccc}
		\vcat{} \colon& \odag & \to & \cat{SymCat}\\
		& \g,\quad \alpha \colon \g \to \g' &\mapsto & \vcat{\g},\quad \vcat{\alpha} \colon \vcat{\g'}\to \vcat{\g}
	\end{array}
	\]
	where: 
	\begin{itemize}
		\item $\vcat{\alpha}$ is defined by sending $v \in V_{\g'}$ to $\bigotimes_{w \in \alpha^{-1}(v)} w$ (ordered according to the order of $\g$), and
		\item $\cat{SymCat}$ is the category of symmetric monoidal categories with symmetric monoidal functors,
	\end{itemize}
		is a contravariant functor.
\end{proposition}
\begin{proof}
	This is analogous to the first part of the proof of \cref{thm:graphhom-cd}, and follows from requiring order-preservation.	
\end{proof}

	The discussion above, together with \cref{prop:bayesian_functor}, justifies the following characterisation.
\begin{proposition}\label{prop:irredundantbayesian_functor}
	There is a bijective correspondence between
	\begin{itemize}
		\item Irredundant Bayesian networks over $\g$, and
		\item Pairs $(\omega,\tau)$, where $\omega\colon \cdsyn{\bullet} \to \finstoch$ is a CD-functor and $\tau\colon \vcat{\g}\to \finstoch$ is a symmetric monoidal functor, for which there exists a CD-functor $F\colon \cdsyn{\g}\to \finstoch$ such that the following diagram commutes:
		\begin{equation}\label{eq:irredundantbayesian_functor}
			\begin{tikzcd}
				& \vcat{\g} \ar[d,hook]\ar[rd,"\tau"]& \\
			\cdsyn{\bullet} \ar[rr,bend right,"\omega"]\ar[r,"!_{\g}"] &\cdsyn{\g}\ar[r, "F", dashed] & \finstoch
			\end{tikzcd}
		\end{equation}
	\end{itemize}
\end{proposition}
In particular, the request of the existence of such an $F$ ensures that $\omega(\bullet)= \bigotimes_{v\in V_{\g}} \tau(v)$.
	\begin{proof}
	Given a probability distribution $\omega$ together with an assignment $\tau$ gives rise to a picture as above: indeed, since $\omega$ factorises as a product $\prod_v f_v$ for some stochastic matrices $f_v$, we can define a CD-functor ${F}\colon \cdsyn{\g}\to \finstoch$ by sending $v\mapsto \tau(v)$ and the generator $\input{parents_morph.tikz}$ to $f_v$. 
	\cref{rem:factorising_distributions_DAG} is sufficient to ensure that the CD-functor $\cdsyn{\bullet} \to \finstoch$ given by $\omega$ indeed factorises through $!_{\g}$. On the other side, we also know that $F$ respects the assignment of sets $\tau$, which can be extended to a symmetric monoidal functor $\vcat{\g}\to \finstoch$, thus ensuring the commutation of the other triangle.
 		
	Conversely, any pair of functors $(\omega,\tau)$ as in the statement leads to a probability distribution $\omega (\begin{tikzpicture}
	\begin{pgfonlayer}{nodelayer}
		\node [style=small box] (0) at (0, 0) {};
		\node [style=none] (1) at (2, 0) {};
	\end{pgfonlayer}
	\begin{pgfonlayer}{edgelayer}
		\draw (1.center) to (0);
	\end{pgfonlayer}
\end{tikzpicture}
)$ that can be factorised, again by \cref{rem:factorising_distributions_DAG}, via $f_v\coloneqq {F}(\input{parents_morph.tikz})$. Note such a choice respects the assignment because $\tau$ coincides with the composition $V(\g)\hookrightarrow \cdsyn{\g}\xrightarrow{F} \finstoch$.
	Therefore, we obtain an irredundant Bayesian network. 
	As the two constructions are clearly the inverse of one another, the wanted bijection holds.
\end{proof}
\begin{remark}\label{rem:vcat_bullet_commuting}
	It may be worth pointing out that the obvious inclusions $\vcat{\g}\hookrightarrow \cdsyn{\g}$ in fact give rise to a natural transformation
	between $\vcat{}$ and the composition $\odag\xrightarrow{\cdsyn{}}\cdcat \to \cat{SymCat}$, where the last functor simply forgets the CD-structure.

	In particular, we may add to \eqref{eq:irredundantbayesian_functor} an additional $\vcat{\bullet}$, which is another way to see why $\omega(\bullet)$ indeed coincides with $\bigotimes_{v \in V_{\g}} \tau(v)$ by means of the following commutativity: 
	\begin{equation*}
			\begin{tikzcd}
				\vcat{\bullet} \ar[r]\ar[d,hook] & \vcat{\g} \ar[d,hook]\ar[rd,"\tau"]& \\
			\cdsyn{\bullet} \ar[rr,bend right,"\omega"]\ar[r,"!_{\g}"] &\cdsyn{\g}  & \finstoch
			\end{tikzcd}
	\end{equation*}
	We stress that this commutativity alone is however insufficient to guarantee the existence of a functor $F$ as in \eqref{eq:irredundantbayesian_functor}. 
	In fact, the existence of such a functor ensures that $\omega$ satisfies certain conditional independencies; without it, we cannot distinguish between different graph structures that share the same set of vertices.
\end{remark}

\begin{example}\label{ex:BEAR!}
To better highlight the connection given by \cref{prop:irredundantbayesian_functor}, let us consider \cref{ex:BEAR}. 
The CD-functor $!_{\g}\colon \cdsyn{\bullet}\to \cdsyn{\g}$ is then defined by sending $\bullet$ to $B\tensor E\tensor A\tensor R$ and $\begin{tikzpicture}
	\begin{pgfonlayer}{nodelayer}
		\node [style=small box] (0) at (0, 0) {};
		\node [style=none] (1) at (2, 0) {};
	\end{pgfonlayer}
	\begin{pgfonlayer}{edgelayer}
		\draw (1.center) to (0);
	\end{pgfonlayer}
\end{tikzpicture}
$ to the string diagram depicted in \cref{fig:ex_bayesian_network}. 
When composing $!_{\g}$ with the CD-functor $\cdsyn{\g}\to \finstoch$ described by the probability tables in \cref{fig:ex_bayesian_network}, we indeed obtain the probability distribution associated to the Bayesian network.
\end{example}

As we already pointed out, our formalism is connected to \newterm{I-maps}.
Given a set of conditional independencies $I$ over some variables, one says that a DAG $\g$ is an \emph{I-map} for $I$ if $I(\g)\subseteq I$, where $I(\g)$ denotes the set of Markov properties associated to $\g$ (local or global depending on the reference).
This choice is subject to a prior assignment $\tau$, sending each vertex to a finite set.
When comparing two DAGs, one may say that $\g$ is an \emph{I-map} for $\g'$ if $I(\g)\subseteq I(\g')$.
In the theory, one tacitly assumes that both graphs are on the same vertices, which forces us to restrict to identity-like homomorphisms. 

\begin{proposition}\label{prop:Imap_DAG}
	Let $\g$ and $\g'$ be ordered DAGs and let $\alpha \colon \g' \to \g$ be a graph homomorphism which is the identity on vertices. 
	Then $\g$ is an I-map for $\g'$. 
\end{proposition}

We remark that not all I-maps are obtained by graph homomorphisms: a simple example is given by \begin{tikzpicture}
	\begin{pgfonlayer}{nodelayer}
		\node [style=bw] (15) at (-1.5, 0) {$A$};
		\node [style=bw] (31) at (1.5, 0) {$B$};
	\end{pgfonlayer}
	\begin{pgfonlayer}{edgelayer}
		\draw [style=arrow] (15) to (31);
	\end{pgfonlayer}
\end{tikzpicture}
 and \begin{tikzpicture}
	\begin{pgfonlayer}{nodelayer}
		\node [style=bw] (15) at (-1.5, 0) {$A$};
		\node [style=bw] (31) at (1.5, 0) {$B$};
	\end{pgfonlayer}
	\begin{pgfonlayer}{edgelayer}
		\draw [style=arrow] (31) to (15);
	\end{pgfonlayer}
\end{tikzpicture}
, whose associated sets of conditional independences are simply the emptyset.

\begin{proof}
	By \cref{thm:graphhom-cd}, the obtained $\cdsyn{\g}\to \cdsyn{\g'}$, together with terminality of $\bullet$ in $\odag$, ensures the factorisation $!_{\g'}\colon \cdsyn{\bullet}\xrightarrow{!_{\g}} \cdsyn{\g}\to \cdsyn{\g'}$. 
	This means that any (irredundant) Bayesian network $(\omega,\tau)$ over $\g'$ is also a Bayesian network over $\g$.
	By the d-separation criterion, this means that $(\omega,\tau)$ satisfies the Markov properties of $\g$ as soon as it satisfies the Markov properties of $\g'$. 
	We can now apply \cite[Thm.~3.4]{koller2009probabilistic} to conclude.	
\end{proof}
\begin{remark}[Decontracting Edges]\label{rem:decontracting}
	A surjective order-preserving graph ho\-mo\-mor\-phism $\alpha \colon \g \to \g'$ can be factored as $\g\to \g'' \to \g'$, where $\g \to \g''$ is the identity on vertices and $\g''\to \g'$ is simply a contraction of complete subgraphs.
	To see how, one simply defines $\g''$ as the graph with the same vertices of $\g$ and an edge $v\to w$ in $\g''$ if $v<w$ and either $\alpha(v)=\alpha(w)$ or $\alpha(v)\to \alpha(w)$.
	Then the identity on vertices $\g \to \g''$ is still a graph homomorphism since the order on $\g$ allows $v\to w$ only when $v<w$.
	Moreover, by definition of $\g''$, it follows from a direct check that the map $\g''\to \g'$ given by $\alpha$ on vertices is an order-preserving graph homomorphism where every preimage is a complete subgraph.

	Using this modification, we can then connect surjective order-preserving graph homomorphisms to graph homomorphisms given by the identity on vertices, and therefore to I-maps by \cref{prop:Imap_DAG}.
\end{remark}

We now turn to Markov networks, where the observation about redundancy (\cref{def:irredundantBN}) is even more relevant. Indeed, the generality of cliques causes a considerable redundancy of information, which is often `pulled under the rug' in applications, but should be made explicit in a mathematically rigorous treatment. 
	
\begin{definition}[Irredundant Markov Network]\label{def:irredundantMN}
	An \newterm{irredundant} Markov network over an (ordered) undirected graph $\h$ is a probability distribution $\omega$, together with an assignment $\tau(v)$ of a finite set for any vertex $v\in V_{\h}$, such that $\omega(V_{\h})=\frac{1}{Z}\prod_{C\in \clique{\h}} \phi_C(x_C)$ for some factors $\phi_C$, where $Z$ is a normalisation coefficient.
\end{definition}

Since $\omega$ is a proper distribution, $Z \neq 0$, so any irredundant network is also non-degenerate. 

\begin{remark}
	The definition above differs from the one we provided in our previous work~\cite{lorenzin2025moralisation}. 
	This discrepancy mirrors the case of Bayesian networks, detailed in \cref{rem:irredundant_distinction}.
\end{remark}

As in the case of Bayesian networks, $\bullet$ is the final object in $\ougr$, yielding a morphism $!_{\h}\colon \syn{\bullet}\to \syn{\h}$. Moreover, there is an obvious functor $\star\colon \cdsyn{\bullet}\to \syn{\bullet}$ that sends the unique generator to itself.

To take care of the normalisation coefficient, a characterisation analogous to the one of \Cref{prop:irredundantbayesian_functor} requires a few extra pieces. 
We also adopt a PROP $\vcat{\h}$ for an undirected graph, analogously to the case of Bayesian networks.\footnote{The definition coincides verbatim with \cref{def:vcat}, since the requirement that $\g$ be directed has no effect in this context. In fact, the definition holds just as well for total ordered (finite) sets.} 
\begin{remark}\label{rem:vcat_functor}
The PROPs $\vcat{\h}$ gives rise to a contravariant functor $\ougr\to \cat{SymCat}$, analogously to \cref{prop:vcat_functor}.
\end{remark}

\begin{proposition}\label{prop:irredundantmarkov_functor}
	There is a bijective correspondence between
	\begin{itemize}
		\item Irredundant Markov networks over $\h$, and 
		\item Pairs $(\omega,\tau)$, where $\omega\colon \cdsyn{\bullet} \to \finstoch$ is a CD-functor and $\tau\colon \vcat{\g}\to \finstoch$ is a symmetric monoidal functor, for which there exists a hypergraph functor $\Phi \colon \syn{\h} \to \mat$ such that the following diagram commutes:
		\begin{equation}\label{eq:mn_factorisation}
		\begin{tikzcd}
			&&\finstoch \ar[drr,"i", bend left=10]& &\\ 
			\cdsyn{\bullet}\ar[urr,"\omega", bend left=10]\ar[rd,"\star"] &&\vcat{\h}\ar[d,hook]\ar[u,"\tau"]&& \finprojstoch\\
			&\syn{\bullet}\ar[r,"!_{\h}"] &\syn{\h}\ar[r,dashed,"\Phi"] & \mat\ar[ur,"q"]&		
		\end{tikzcd}
	\end{equation}
	\end{itemize}
\end{proposition}
	\begin{proof}
		If a probability distribution $\omega$ and an assignment $\tau$ give rise to an irredundant Markov network, then we can find factors $\phi_C\colon I \to C$ such that $\omega(V_{\h})= \frac{1}{Z} \prod_C \phi_C(C)$, where $Z$ is a normalisation coefficient.
		Then, we define a hypergraph functor ${\Phi}\colon \syn{\h}\to \mat$ by sending $v$ to $\tau(v)$ and $\begin{tikzpicture}
	\begin{pgfonlayer}{nodelayer}
		\node [style=small box] (0) at (0, 0) {};
		\node [style=none] (1) at (1.5, 0) {};
		\node [style=none] (2) at (2, 0) {$C$};
	\end{pgfonlayer}
	\begin{pgfonlayer}{edgelayer}
		\draw (1.center) to (0);
	\end{pgfonlayer}
\end{tikzpicture}
$ to $\phi_C$. This respects the assignment $\tau$ by construction, which can be interpreted as a symmetric monoidal functor $\vcat{\h}\to \finstoch$.
		Moreover, $\omega$ gives rise to a CD-functor $\cdsyn{\bullet}\to \finstoch$ by $\begin{tikzpicture}
	\begin{pgfonlayer}{nodelayer}
		\node [style=small box] (0) at (0, 0) {};
		\node [style=none] (1) at (2, 0) {};
	\end{pgfonlayer}
	\begin{pgfonlayer}{edgelayer}
		\draw (1.center) to (0);
	\end{pgfonlayer}
\end{tikzpicture}
\mapsto \omega$.
		The commutativity of \eqref{eq:mn_factorisation} holds as it corresponds to $\omega(V_{\h})\propto \prod_C \phi_C(C)$, and this is true by definition of an irredundant Markov network.

		Conversely, let us call $Q$ the composition $\cdsyn{\bullet}\to \syn{\bullet}\to \syn{\h}\xrightarrow{\Phi} \mat$. Then commutativity of the outside diagram of \eqref{eq:mn_factorisation} (i.e., the one disregarding $\vcat{\h}$) is exactly requiring that $Q\propto \omega$, which holds if and only if there exists a normalisation coefficient $Z$ such that $\omega = \frac{1}{Z}Q$ (to prevent notation overload, here we use $Q$ and $\omega$ both as the functors and their associated distributions $Q(\begin{tikzpicture}
	\begin{pgfonlayer}{nodelayer}
		\node [style=small box] (0) at (0, 0) {};
		\node [style=none] (1) at (2, 0) {};
	\end{pgfonlayer}
	\begin{pgfonlayer}{edgelayer}
		\draw (1.center) to (0);
	\end{pgfonlayer}
\end{tikzpicture}
)$ and $\omega(\begin{tikzpicture}
	\begin{pgfonlayer}{nodelayer}
		\node [style=small box] (0) at (0, 0) {};
		\node [style=none] (1) at (2, 0) {};
	\end{pgfonlayer}
	\begin{pgfonlayer}{edgelayer}
		\draw (1.center) to (0);
	\end{pgfonlayer}
\end{tikzpicture}
)$). Moreover, the commutativity request on $\vcat{\h}$ ensures that $\Phi$ respects the assignment $\tau$.
		By \cref{rem:factorising_distributions_UGr}, we have that $Q(V_{\h}) = \prod_C \phi_C (C)$ with $\phi_C\coloneqq \Phi(\begin{tikzpicture}
	\begin{pgfonlayer}{nodelayer}
		\node [style=small box] (0) at (0, 0) {};
		\node [style=none] (1) at (1.5, 0) {};
		\node [style=none] (2) at (2, 0) {$C$};
	\end{pgfonlayer}
	\begin{pgfonlayer}{edgelayer}
		\draw (1.center) to (0);
	\end{pgfonlayer}
\end{tikzpicture}
)$ because of the factorisation, hence $\omega(V_{\h}) = \frac{1}{Z} \prod_C \phi_C (C)$, and thus $\omega$ and $\tau$ give an irredundant Markov network.
	\end{proof}

\begin{example}
To better understand \cref{prop:irredundantmarkov_functor}, let us consider \cref{ex:misconception}.
The hypergraph functor $!_{\h}\colon \syn{\bullet}\to \syn{\h}$ sends $\bullet$ to $A\tensor B \tensor C \tensor D$ and $\begin{tikzpicture}
	\begin{pgfonlayer}{nodelayer}
		\node [style=small box] (0) at (0, 0) {};
		\node [style=none] (1) at (2, 0) {};
	\end{pgfonlayer}
	\begin{pgfonlayer}{edgelayer}
		\draw (1.center) to (0);
	\end{pgfonlayer}
\end{tikzpicture}
$ to the string diagram in \cref{fig:markov_network_ex}.
The unnormalised distribution $Q$ is then described by the composition $\syn{\bullet}\xrightarrow{!} \syn{\h}\to \mat$, where $\syn{\h}\to \mat$ is determined by \eqref{tab:misconception_factors}.
\end{example}

\begin{remark}\label{rem:vcat_bullet_commuting_mn}
The same idea of \cref{rem:vcat_bullet_commuting} can be applied also for irredundant Markov networks. Explicitly, 
	\begin{itemize}
	\item $\vcat{\h}\hookrightarrow \syn{\h}$ gives rise to a natural transformation;
	\item We can add $\vcat{\bullet}$ in \eqref{eq:mn_factorisation} to ensure that $\omega(\bullet) = \bigotimes_{v \in V_{\h}} \tau(v)$ even without the existence of $\Phi$, but again this is insufficient to ensure the existence of $\Phi$.
	\end{itemize}
\end{remark} 

We now want to prove a statement regarding I-maps for ordered undirected graphs.
To this aim, we consider the global Markov properties $I(\h)$ associated to the undirected graph $\h$. 
Then $\h$ is an I-map for $\h'$ if $I(\h)\subseteq I(\h')$.
\begin{proposition}\label{prop:Imap_UGr}
	Let $\h$ and $\h'$ be ordered undirected graphs and let $\alpha \colon \h'\to \h$ be a graph homomorphism which is the identity on vertices.
	Then $\h$ is an I-map for $\h'$.
\end{proposition}
\begin{proof}
	By \cref{thm:graphhom-hyp}, we have a hypergraph functor $\syn{\alpha}\colon \syn{\h}\to \syn{\h'}$ allowing a factorisation $!_{\h'} \colon \syn{\bullet}\to \syn{\h}\to \syn{\h'}$.
	Therefore, any (irredundant) Markov network $(\omega,\tau)$ over $\h'$ is also a Markov network over $\h$, and the latter satisfies the global Markov properties $I(\h)$ (\cite[Thm.~4.1]{koller2009probabilistic}).
	The generality of $(\omega,\tau)$ then ensures that $I(\h)\subseteq I(\h')$ by \cite[Thm.~4.3]{koller2009probabilistic}.
\end{proof}
\begin{remark}
	The strategy employed in \cref{rem:decontracting} can be adapted to decontract edges also in the undirected case. 
	In particular, surjective order-preserving graph homomorphisms can be connected to I-maps as explained there.
\end{remark}

\section{Morphisms Between Networks}\label{sec:categoriesnetworks}

In order to define functorial transformations between Bayesian and Markov networks, we need a full definition of the categories involved. 
The characterisations of \Cref{sec:networks,sec:irredundant_networks} only provide the objects of these categories. 
We now identify a suitable notion of morphisms, culminating in the definitions of the categories of Bayesian networks and Markov networks (\cref{def:bn_cat,def:mn_cat}). 
As our aim is to describe moralisation and triangulation, and these modifications pertain specifically to the study of conditional independencies, the most natural approach is to focus attention to irredundant Bayesian and Markov networks.
We will discuss this further in \cref{rem:moralisation_and_strong_commutativity}. 

Given the functorial perspective on networks, a natural candidate for morphisms are compatible pairs of a `morphism between syntaxes' and a `morphism between semantics'. The former will simply be an order-preserving graph homomorphism, in view of \cref{thm:graphhom-cd,thm:graphhom-hyp}. For the latter, recall that both irredundant Bayesian and Markov networks are defined by probability distributions factoring through a certain structure. In $\finstoch$ we may regard two such distributions as maps $\omega\colon I \to X$ and $\omega'\colon I \to Y$, and a morphism between them as a stochastic matrix $f\colon X \to Y$ such that $\omega' = f\comp \omega$. Now, in the functorial perspective $\omega,\omega'$ are identified with CD-functors $\cdsyn{\bullet} \to \finstoch$. We can lift the notion of morphism between distributions from $\finstoch$ to the level of such functors, as follows.
\begin{definition}
	Let $F,G \colon \cC \to \cD$ be two symmetric monoidal functors.
	A \newterm{monoidal transformation} $\eta \colon F \to G$ is a family of morphisms $\eta_X\colon F(X) \to G(X)$ for every object $X \in \cC$ such that $\eta_{X\otimes Y}=\eta_X \otimes \eta_Y$ and $\eta_I =\id_I$.

	If $\cC=\cdsyn{\bullet}$, then $\eta$ is \newterm{distribution-preserving} if $\eta_{\bullet} \omega(\begin{tikzpicture}
	\begin{pgfonlayer}{nodelayer}
		\node [style=small box] (0) at (0, 0) {};
		\node [style=none] (1) at (2, 0) {};
	\end{pgfonlayer}
	\begin{pgfonlayer}{edgelayer}
		\draw (1.center) to (0);
	\end{pgfonlayer}
\end{tikzpicture}
) = \omega' (\begin{tikzpicture}
	\begin{pgfonlayer}{nodelayer}
		\node [style=small box] (0) at (0, 0) {};
		\node [style=none] (1) at (2, 0) {};
	\end{pgfonlayer}
	\begin{pgfonlayer}{edgelayer}
		\draw (1.center) to (0);
	\end{pgfonlayer}
\end{tikzpicture}
)$.
\end{definition}

Note that we did not require naturality of the transformation, as one may initially expect: such a requirement is too strong, as it only holds when $f$ is a deterministic function (each column has exactly one non-zero value). To justify our notion, observe the following.

\begin{lemma}
	Let $\omega,\omega' \colon \cdsyn{\bullet}\to \finstoch$.
	There is a bijective correspondence between stochastic matrices $f$ such that $\omega'(\begin{tikzpicture}
	\begin{pgfonlayer}{nodelayer}
		\node [style=small box] (0) at (0, 0) {};
		\node [style=none] (1) at (2, 0) {};
	\end{pgfonlayer}
	\begin{pgfonlayer}{edgelayer}
		\draw (1.center) to (0);
	\end{pgfonlayer}
\end{tikzpicture}
) = f\comp \omega (\begin{tikzpicture}
	\begin{pgfonlayer}{nodelayer}
		\node [style=small box] (0) at (0, 0) {};
		\node [style=none] (1) at (2, 0) {};
	\end{pgfonlayer}
	\begin{pgfonlayer}{edgelayer}
		\draw (1.center) to (0);
	\end{pgfonlayer}
\end{tikzpicture}
) $ and distribution-preserving monoidal transformations $\eta \colon \omega\to \omega'$.
\end{lemma}
\begin{proof}
	This is a direct check since the objects of $\cdsyn{\bullet}$ are of the form $\bullet^{\tensor n}$ for some $n$, and thus $\eta_{\bullet^{\otimes n}}=f^{\otimes n}$ describes such a bijection.
\end{proof}



We now aim to discuss what morphisms between irredundant networks should be. For simplicity, we focus on Bayesian networks, though similar observations carry over to the case of Markov networks. 
In particular, from \cref{prop:irredundantbayesian_functor}, we will write an irredundant Bayesian network as a triple $(\omega,\g,\tau)$, where $\omega\colon \cdsyn{\bullet}\to \finstoch$ identifies a probability distribution, $\g$ a DAG ensuring a factorisation, and $\tau\colon \vcat{\g}\to \finstoch$ the assignment to the vertices.

We briefly recall that in \cite{lorenzin2025moralisation} we defined a morphism between irredundant Bayesian networks as a pair $(\alpha,\eta)$, where $\alpha$ is a morphism in $\odag$ and $\eta$ is a distribution-preserving monoidal transformation between the two distributions $\omega,\omega'\colon \cdsyn{\bullet}\to \finstoch$. 
In other words, one takes care of the DAGs, i.e.\ the syntax, while the other focuses on semantics.

However, these modifications are not mindful of the assignments $\vcat{\g}\to \finstoch$ (recall that in \cite{lorenzin2025moralisation} assignments were not considered; see \cref{rem:irredundant_distinction}). 
In particular, if we choose to keep the same definition of morphisms, changing the assignment will lead to isomorphic Bayesian networks (via the obvious pair $(\alpha,\eta)=(\id,\id)$).
To see why that is somehow peculiar and undesirable, consider the BEAR Bayesian network of \cref{ex:BEAR}. 
An alternative assignment that still respects the same DAG is given by setting $B\coloneqq I$ and $A\coloneqq B \otimes A$ (i.e., we now consider the burglar and the alarm together, so that both will be caused by the earthquake). 
Then, under the current formalism, this extravagant assignment will be indistinguishable from the standard one.

This need for distinction uncovers a subtle question: which part should be modified according to the assignments? The graph homomorphisms or the distributions?
More explicitly, let $(\omega,\g,\tau)$ and $(\omega',\g',\tau)$ and consider a graph homomorphism $\alpha \colon \g' \to \g$ and a distribution-preserving monoidal transformation $\eta \colon \omega \to \omega'$. 
The most natural requests may be either that $\tau'\comp \vcat{\alpha} = \tau$, or that $\eta_{\bullet}$ factors as a tensor product according to $\g$, or rather $V_{\g}$, so that
\begin{equation}\label{eq:eta_BN}
\eta_{\bullet} = \bigotimes_{v \in V_{\g}} \eta_v, \qquad \text{where } \eta_v \colon \tau(v) \to \tau' \vcat{\alpha}(v).	
\end{equation}
While far from asserting any definitive conclusion, we lean toward the latter, as it reveals an interaction between the choices of $\eta$ and $\alpha$.

Finally, we are ready to define the two categories of networks, where the choice of objects is justified by~\cref{prop:irredundantbayesian_functor,prop:irredundantmarkov_functor}.

\begin{definition}\label{def:bn_cat}
	The \newterm{category of (irredundant) Bayesian networks} $\bn$ is defined as:
\begin{itemize}
	\item The objects are triples $(\omega , \g,\tau)$, where $\omega$ is a CD-functor $\cdsyn{\bullet} \to \finstoch$, $\g$ is an ordered DAG, and $\tau\colon \vcat{\g}\to \finstoch$ is a symmetric monoidal functor for which there exists a CD-functor $F \colon \cdsyn{\g}\to \finstoch$ making the following commute:
	\begin{equation*}
		\begin{tikzcd}
			& \vcat{\g} \ar[d,hook]\ar[rd,"\tau"]& \\
		\cdsyn{\bullet} \ar[rr,bend right,"\omega"]\ar[r,"!_{\g}"] &\cdsyn{\g}\ar[r, "F", dashed] & \finstoch
		\end{tikzcd}
	\end{equation*}
	\item A morphism $(\omega,\g,\tau)\to (\omega',\g',\tau')$ is a triple $(\alpha,\eta,\hat{\eta})$, where $\alpha$ is a morphism $\alpha\colon \g' \to \g$  in $\odag$, $\eta$ is a monoidal transformation $\tau \to \tau' \vcat{\alpha}$, and $\hat{\eta}\colon \omega \to \omega'$ is a distribution-preserving monoidal transformation, and moreover $\eta$ and $\hat{\eta}$ coincide as transformations between functors $\vcat{\bullet}\to \finstoch$; more precisely, the following commutes: 
\begin{equation}\label{eq:eta_hateta}
\begin{tikzcd}[row sep=large]
\vcat{\bullet}\ar[r]\ar[d,hook] & \vcat{\g} \ar[d,bend right=40, "\tau" {name=T,left}]\ar[d,bend left=40, "\tau'\comp \vcat{\alpha}" {name=Tprime,right}]\\
\cdsyn{\bullet}\ar[r,bend right=40,"\omega" {name=W,below},end anchor=205]\ar[r,bend left=40, "\omega'" {name=Wprime,above},end anchor=150] & \finstoch
\arrow[from=W, to=Wprime, Rightarrow, shorten=2mm,"\hat{\eta}"] \arrow[from=T, to=Tprime, Rightarrow, shorten=2mm, "\eta"]
\end{tikzcd}
\end{equation}
	Composition is component-wise.
\end{itemize}
\end{definition}
First of all, we note that $\eta$ completely determines $\hat{\eta}$, so we could alternatively say that $\eta$, when considered as a monoidal transformation $\omega \to \omega'$, is distribution-preserving.

Concretely, $\eta$ and $\hat{\eta}$ give rise to a tensor product of stochastic matrices as in \eqref{eq:eta_BN}, i.e.\ such that 
\[ \left(\bigotimes_{v \in V_{\g}} \eta_v \right) \omega = \hat{\eta}_{\bullet} \omega = \omega'.\] 
This is immediate by unpacking the definition.

\begin{definition}\label{def:mn_cat}
	The \newterm{category of (irredundant) Markov networks} $\mn$ is defined as:
	\begin{itemize}
		\item The objects are triples $(\omega , \h,\tau)$, where $\omega$ is a CD-functor $\cdsyn{\bullet} \to \finstoch$, $\h$ is an undirected graph, and $\tau \colon \vcat{\h}\to \finstoch$ is a symmetric monoidal functor admitting a factorisation according to~\eqref{eq:mn_factorisation}. 
		\item 
		A morphism $(\omega,\h,\tau)\to (\omega',\h',\tau')$ is a triple $(\alpha,\eta,\hat{\eta})$, where $\alpha$ is a morphism $\alpha\colon \h' \to \h$  in $\odag$, $\eta$ is a monoidal transformation $\tau \to \tau' \vcat{\alpha}$, and $\hat{\eta}\colon \omega \to \omega'$ is a distribution-preserving monoidal transformation, and moreover $\eta$ and $\hat{\eta}$ coincide as transformations between functors $\vcat{\bullet}\to \finstoch$, as expressed in \eqref{eq:eta_hateta}. Composition is component-wise.
	\end{itemize}	
\end{definition}

The chosen directionality of morphisms captures the process of \emph{revealing structure and updating}.
To support this claim, note that for any distribution $\omega$ equipped with an assignment $\tau$, a morphism $(\omega,\bullet,\tau)\to (\omega,\g,\tau)$ provides more information about $\omega$ in the form of a Bayesian network structure. 
In particular, from a syntactic perspective, we can reveal structure. 
Instead, semantics allows us to update the system without ``changing our beliefs'', formally captured by the preservation of graph structure. 
For a better understanding, we present the following result, which identifies a broad class of monoidal transformations that can safely serve as semantics morphisms in $\bn$ (and $\mn$). 

\begin{proposition}\label{prop:pearlupdate}
	Let $(\omega,\g,\tau)$ be a Bayesian network. Then 
	\[ \eta_v\coloneqq\quad \input{etav_ok.tikz}\] 
	where $\pi_v$ is an isomorphism, gives rise to a morphism $(\id,\eta,\hat{\eta})\colon (\omega,\g,\tau)\to (\omega',\g,\tau')$, where $\omega'\coloneqq \left(\bigotimes_v \eta_v\right) \comp \omega$ and $\tau'(v)\coloneqq A_v\otimes B_v$. 

	The same holds for Markov networks.
\end{proposition}

Roughly speaking, the proposition states that the network structure is preserved when we have a reversible operation ($\pi_v$) paired with a ``coloured predicate'' ($p_v$). 
This is similar to Pearl's updates as studied by Jacobs~\cite{jacobs2019mathematics} (see also \cite[Def.~21]{tull2024activeinference} for a string diagrammatic description).

The proof actually holds under the assumptions of \cref{set:conditionals}.
\begin{proof}
	We recall that for any isomorphism $\pi$, the fact that $\finstoch$ has conditionals (\cref{def:conditionals}) ensures the following equality
	\begin{equation}\label{eq:positivity}
	\input{pi_copy.tikz}\quad =\quad \input{pi_copy_inverse.tikz}
	\end{equation}
	by \cite[Lem.~11.24]{fritz2019synthetic}.

	By induction on the number of vertices, we assume that the statement holds true for graphs with $n-1$ vertices, and prove it for graphs with $n$ vertices. Given a DAG $\g$, we take the biggest element $v \in V_{\g}$. We define $\Sigma \coloneqq \Sigma_{\g}\setminus \lbrace{f_v}\rbrace$, $\pi\coloneqq \bigotimes_{w \in V_{\g} \setminus \lbrace v\rbrace } \pi_w$, and $p\coloneqq  \bigotimes_{w \in V_{\g} \setminus \lbrace v\rbrace } p_w$. Then 
	
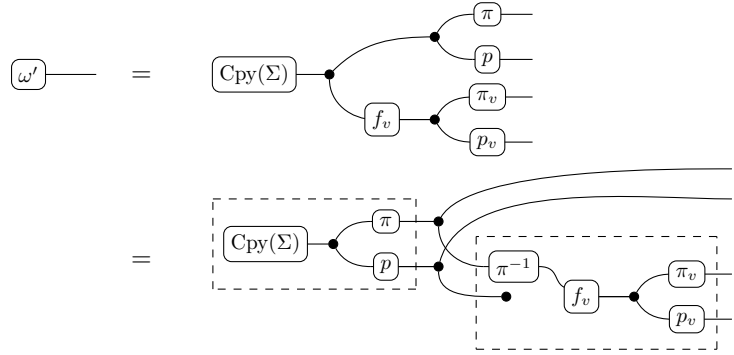
\[
		\begin{aligned}
			\begin{tikzpicture}
	\begin{pgfonlayer}{nodelayer}
		\node [style=none] (65) at (1.25, 0) {};
		\node [style=small box] (66) at (-1, 0) {$\omega'$};
	\end{pgfonlayer}
	\begin{pgfonlayer}{edgelayer}
		\draw (65.center) to (66);
	\end{pgfonlayer}
\end{tikzpicture}
 \quad &= \qquad {\input{beliefupdate_1.tikz}}\\
			&= \qquad {\input{beliefupdate_2.tikz}}
		\end{aligned}
	\]
	where the dashed box on the right gives the new definition of $f_v$, whereas the dashed box on the left concludes the statement by induction hypothesis. Note that \eqref{eq:positivity} ensures the validity of the second equality.

	The situation of Markov networks can be treated similarly by means of the Frobenius equations \eqref{eq:hypergraph}. 
\end{proof}

Another important aspect of our morphisms is the fact that they can be split in two parts: one that is completely syntactic and one that is completely semantic.
\begin{proposition}
	In $\bn$, each morphism $(\alpha,\eta,\hat{\eta})\colon (\omega,\g,\tau)\to (\omega',\g',\tau')$ can be decomposed as 
	\[
	\begin{tikzcd}[column sep=large]
		(\omega,\g,\tau) \ar[r,"{(\id,\eta ,\hat{\eta})}"] & (\omega',\g,\tau'\vcat{\alpha})\ar[r,"{(\alpha,\id,\id)}"] & (\omega',\g',\tau')
	\end{tikzcd}
	\]
	The analogous statement holds true for $\mn$.
\end{proposition}
\begin{proof}
	We note that as soon as the types check, then the composition is the expected one, since by definition it is component-wise. 
	Regarding $(\id,\eta,\hat{\eta})$, we see that this does not change the graph, while $\eta \colon \tau \to \tau'\vcat{\alpha}$ and $\hat{\eta}\colon \omega \to \omega'$ indeed respect the types. 
	To check that $(\alpha,\id,\id)$ has the right type, the only nontrivial task is to note that the second argument of the morphism corresponds to $\id \colon \tau' \vcat{\alpha}\to \tau' \vcat{\alpha}$.
	The reasoning applies verbatim to $\mn$.
\end{proof}



\section{Moralisation as a Functor}\label{sec:moralisation}

We now have all the ingredients to discuss network transformations.
We begin by focusing on moralisation, which transforms a Bayesian network $\g$ into a Markov network $\mor{\mathcal{G}}$, with the property that a distribution $\omega$ factorising through $\g$ also factorises through $\mor{\mathcal{G}}$~\cite{koller2009probabilistic}.

\begin{definition}\label{def:moralisation}
Let $\g$ be a DAG. Its \newterm{moralisation} $\mor{\mathcal{G}}$ is the undirected graph whose vertices are the same as $\mathcal{G}$ and there is an edge between $v$ and $w$ (written $v \edge w$) whenever  in $\g$ there is an edge between them, or they are both parents of the same vertex.
\end{definition}

We now discuss how to turn moralisation into a functor $\bn \to \mn$. 
Recall that an (irredundant) Bayesian network, that is, an object of $\bn$, is a triple $(\omega, \g,\tau)$ with $\omega \colon \cdsyn{\bullet} \to \finstoch$ a CD-functor, $\g$ an ordered DAG, and $\tau\colon \vcat{\g}\to \finstoch$ a symmetric monoidal functor. 
	The corresponding moralisation, an object of $\mn$, is going to be defined as $(\omega,\mor{\g},\tau)$, where $\omega$ should factorise through $\syn{\mor{\g}}$ according to \eqref{eq:mn_factorisation}, instantiated as follows.
	\begin{equation}\label{eq:factoringdistribution_maintext}
		\begin{tikzcd}[row sep=0.1pt]
			& \syn{\mor{\g}} \ar[r] &\mat\ar[dr,"q"]&  \\
			\syn{\bullet} \ar[ru,"!"] &&& \finprojstoch\\
			& \cdsyn{\bullet} \ar[lu,"\star"]\ar[r,"\omega"] & \finstoch \ar[ur,"i"]&
		\end{tikzcd}
	\end{equation}
	The outstanding question is how to correctly define $\syn{\mor{\g}} \to \mat$ in~\eqref{eq:factoringdistribution_maintext}. This goes in three steps. 
	
	First, we construct $\syn{\g} \coloneqq \freehyp{V_{\g}, \Sigma_{\g}}$, the free \emph{hypergraph} category on $\g$, which comes with an induced $\star_{\g}\colon \cdsyn{\g}\to \syn{\g}$ given by the identity on the generators $\Sigma_{\g}$.
In particular, the family $(\star_{\g})_{\g}$ is a natural transformation when we intepret $\cdsyn{}$ and $\syn{}$ as functors $\odag \to \cat{CDcat}$.
Intuitively, this amounts to taking an ``indirect'' perspective on the data of $\g$. 

Second, observe that $F \colon \cdsyn{\g} \to \finstoch$ also yields a hypergraph functor $\tilde{F} \colon \syn{\g} \to \mat$, defined on generators the same way as $F$.
We note that by definition, this modification satisfies the commutative diagram
\begin{equation}\label{eq:star_nat}
	\begin{tikzcd}
		\cdsyn{\g} \ar[r,"\star_{\g}"]\ar[d,"F"] & \syn{\g}\ar[d,"\tilde{F}"]\\
		\finstoch \ar[r,hook] & \mat
	\end{tikzcd}
\end{equation}

With the same strategy, we define $\tilde{!}_{\g}\colon \syn{\bullet}\to \syn{\g}$, the hypergraph functor such that $\cdsyn{\bullet}\xrightarrow{\star} \syn{\bullet}\xrightarrow{\tilde{!}_{\g}} \syn{\g}$ coincides with $\cdsyn{\bullet} \xrightarrow{!_{\g}}\cdsyn{\g} \xrightarrow{\star_{\g}} \syn{\g}$.

Third, we define $m \colon \syn{\mor{\g}} \to \syn{\g}$ as the hypergraph functor freely obtained by the following mapping on the generators of $\syn{\mor{\g}}$:
\begin{equation}\label{eq:moralisation_functor}
	{\begin{tikzpicture}
	\begin{pgfonlayer}{nodelayer}
		\node [style=small box] (0) at (0, 0) {};
		\node [style=none] (1) at (1.5, 0) {};
		\node [style=none] (2) at (2, 0) {$C$};
	\end{pgfonlayer}
	\begin{pgfonlayer}{edgelayer}
		\draw (1.center) to (0);
	\end{pgfonlayer}
\end{tikzpicture}
} \qquad \longmapsto\qquad \begin{cases}
		{\input{graph_parents_morph.tikz}} & \text{if }C=\lbrace v \rbrace \cup \parents{v} \text{ for some }v \in \g\\
		{\input{omni_clique.tikz}} & \text{otherwise}
	\end{cases}
\end{equation}
This simple description mimics at the level of string diagrammatic syntax the transformation described by \cref{def:moralisation}. Putting these all together, we obtain the desired hypergraph functor $\syn{\mor{\g}}\xrightarrow{m} \syn{\g} \xrightarrow{\tilde{F}} \mat$, thus completing the definition of the Markov network $(\omega, \mor{\g},\tau)$ given by commutativity of~\eqref{eq:factoringdistribution_maintext}. More explicitly, this amounts to the following commutative diagram
\begin{equation}\label{eq:factoringdistribution1}
	\begin{tikzcd}
		\syn{\mor{\g}}\ar[r,"m"]& \syn{\g} \ar[r,"\tilde{F}"] &\mat\ar[dr,"q"]&  \\
		\syn{\bullet} \ar[u,"!_{\mor{\g}}"]\ar[ru,"\tilde{!}_{\g}"] &\cdsyn{\g}\ar[u,"\star_{\g}"]\ar[dr,"F"] && \finprojstoch\\
		& \cdsyn{\bullet} \ar[lu,"\star_{\bullet}"]\ar[r,"\omega"]\ar[u,"!_{\g}"] & \finstoch\ar[uu,hook]\ar[ur,"i"]&
	\end{tikzcd}
\end{equation}
Observe that the only non-obvious step in this construction is the definition of $m \colon \syn{\mor{\g}} \to \syn{\g}$. This is the key piece of our functorial view of moralisation, providing two insights: first, moralisation may be decomposed into an inductive definition on the string diagrammatic syntax capturing the graph structures; second, by regarding networks themselves as functors, moralisation may be simply defined by functor precomposition.
In particular, the proof of functoriality (\cref{thm:moralisation}) will mostly consist in showing commutativity of the upper left triangle in \eqref{eq:factoringdistribution1}, as the rest is clear by construction. Before delving into the details, we offer a description of the functor $m$ in our working example (\cref{ex:BEAR}) and provide two remarks.

\begin{example}
	The moralisation of the DAG in \cref{ex:BEAR} is given by adding an additional edge between $B$ and $E$, since they are both parents of $A$. 
	\[
	\begin{tikzpicture}
	\begin{pgfonlayer}{nodelayer}
		\node [style=bw] (15) at (-0.5, 2) {$B$};
		\node [style=bw] (18) at (1.5, -2) {$A$};
		\node [style=bw] (29) at (3.5, 2) {$E$};
		\node [style=bw] (31) at (5.5, -2) {$R$};
	\end{pgfonlayer}
	\begin{pgfonlayer}{edgelayer}
		\draw (15) to (18);
		\draw (29) to (18);
		\draw (29) to (31);
		\draw (29) to (15);
	\end{pgfonlayer}
\end{tikzpicture}
 \qquad \text{``}\longmapsfrom\text{''} \qquad \input{BEAR_network.tikz}
	\]
	Then $m \colon \syn{\mor{\g}} \to \syn{\g}$ maps the string diagram in $\syn{\mor{\g}}$ representing the moralised network, below left, to the string diagram representing the Bayesian network, below right:
	\begin{equation*}
		\scalebox{0.95}{\input{BEAR_moral2.tikz}}
	\end{equation*}
	where we used $\phi$ to denote the generators in $\syn{\mor{\g}}$ and $f$ for the generators in $\syn{\g}$.
	Postcomposing $m$ with $\syn{\g}\to \mat$ sets $\phi_{BA}$, $\phi_{EA}$, $\phi_B$, $\phi_E$ and $\phi_{ER}$ as the stochastic matrices of the Bayesian network, while the other factors are all set to be constantly one.
\end{example}

\begin{remark}\label{rem:CDUbad} It is worth observing that $\star_{\g}\colon \cdsyn{\g}\to \syn{\g}$ would not be well-defined if we had considered the free \emph{CDU-category} (see \cite[Def.~2.1]{jacobs2019causal_surgery}) in place of $\cdsyn{\g}$, because the addition of the generator $\begin{tikzpicture}
	\begin{pgfonlayer}{nodelayer}
		\node [style=bn] (24) at (-0.75, 0) {};
		\node [style=none] (25) at (0.75, 0) {};
		\node [style=none] (26) at (1.25, 0) {$X$};
	\end{pgfonlayer}
	\begin{pgfonlayer}{edgelayer}
		\draw (24) to (25.center);
	\end{pgfonlayer}
\end{tikzpicture}
$ and the equation $\begin{tikzpicture}
	\begin{pgfonlayer}{nodelayer}
		\node [style=bn] (35) at (0.75, 0) {};
		\node [style=bn] (36) at (-0.75, 0) {};
	\end{pgfonlayer}
	\begin{pgfonlayer}{edgelayer}
		\draw (35) to (36);
	\end{pgfonlayer}
\end{tikzpicture}
 = \id_I$ would have disrupted the functoriality of $\mathsf{CDUSyn}_{\g} \to \syn{\g}$. This is why we opted for CD-categories to model Bayesian networks, unlike in~\cite{jacobs2019causal_surgery}, which uses CDU-categories because the extra generator $\begin{tikzpicture}
	\begin{pgfonlayer}{nodelayer}
		\node [style=bn] (24) at (-0.75, 0) {};
		\node [style=none] (25) at (0.75, 0) {};
		\node [style=none] (26) at (1.25, 0) {$X$};
	\end{pgfonlayer}
	\begin{pgfonlayer}{edgelayer}
		\draw (24) to (25.center);
	\end{pgfonlayer}
\end{tikzpicture}
$ is needed to model causal intervention.\end{remark}

\begin{remark}[Nonfunctoriality of Redundant Networks]\label{rem:moralisation_and_strong_commutativity}
At the beginning of \cref{sec:irredundant_networks}, we discussed how irredundant networks should be preferred, since we want to focus on conditional independencies.
There, we stated that such a choice is more natural in our setting, without providing a formal justification. 
We can now support this claim by showing that the natural functor $m\colon \syn{\mor{\g}}\to \syn{\g}$, defined in \cref{eq:moralisation_functor}, does not respect the syntax of redundant networks.

More explicitly, let $\alpha \colon \g' \to \g$ be a graph homomorphism. It is easy to check that it is also a graph homomorphism $\mor{\g'}\to \mor{\g}$, and we may wonder whether
\[
\begin{tikzcd}
	\syn{\mor{\g}} \ar[r,"\syn{\alpha}"]\ar[d,"m"] & \syn{\mor{\g'}}\ar[d, "m"]\\
	\syn{\g}\ar[r,"\syn{\alpha}"] & \syn{\g'} 
\end{tikzcd}
\]
commutes. 
In other words, this commutativity would ensure that the syntax of redundant networks is preserved.

As hinted at above, this is generally not the case. For example, take the homomorphism described in \cref{subfig:graphhoms_bn_ex1} and its moralisation (\cref{subfig:graphhoms_mn_ex1}). Then the following diagrams do not commute:
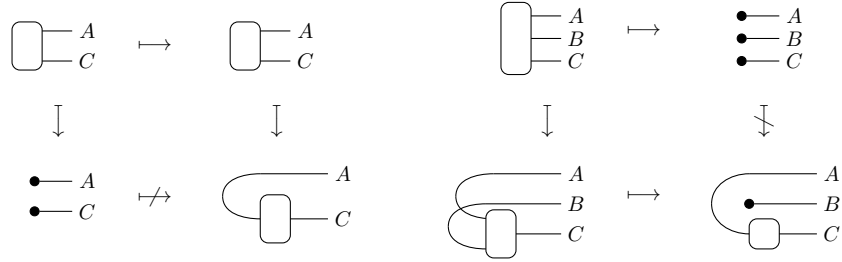
\[
\input{moralisation_strong_commutativity.tikz}
\]
Note this issue does not occur for irredundant networks, and we are able to prove the wanted functoriality (see \cref{thm:moralisation} below), since a network structure is required to exist but not fixed beforehand, and therefore such a strong commutativity condition is not necessary.
\end{remark}

We now proceed to prove the functoriality of moralisation. 
We start by a preliminary definition.

\begin{definition}
	Let $\cC$ be a hypergraph category. Given a morphism $f\colon X \to Y$, its \newterm{graph} is defined as $\input{graph_f_def.tikz}\coloneqq \input{fcapXY.tikz}$.
	For a finite set of morphisms $S$, we also write $\graph{S}$ for the set of graphs.
\end{definition}

\begin{lemma}\label{lem:graphcopycomp}
	Let $\g$ be an ordered DAG and let $S\subseteq \Sigma_{\g}$ be a set of generators.
	Then, in $\syn{\g}$, we have $\graph{\copycomp{S}}=\compcomp{\graph{S}}$.
\end{lemma}
In the statement we explicitly mention $\syn{\g}$ to point out that we cannot consider the case of CD-categories, where the compare-composition is not defined.
The proof is by induction, and it follows from the special Frobenius equations~\eqref{eq:hypergraph}.

\begin{theorem}\label{thm:moralisation}
	Moralisation gives rise to a functor $\mor{-}\colon \bn \to \mn$ which on objects maps ${(\omega,\g,\tau)}$ to $(\omega,\mor{\g},\tau)$.
\end{theorem} 

\begin{proof}
	To ensure that the moralisation gives rise to a functor, we first prove that $\tilde{!}_{\g}\colon \syn{\bullet}\to \syn{\g}$ factors through $m\colon \syn{\mor{\g}}\to \syn{\g}$.
	Recall that $m$, defined in~\eqref{eq:moralisation_functor}, sends $\begin{tikzpicture}
	\begin{pgfonlayer}{nodelayer}
		\node [style=small box] (0) at (0, 0) {};
		\node [style=none] (1) at (1.5, 0) {};
		\node [style=none] (2) at (2, 0) {$C$};
	\end{pgfonlayer}
	\begin{pgfonlayer}{edgelayer}
		\draw (1.center) to (0);
	\end{pgfonlayer}
\end{tikzpicture}
$ to the graph $\input{graph_parents_morph.tikz}$ whenever $C=\lbrace v \rbrace \cup \parents{v}$, and to $\begin{tikzpicture}
	\begin{pgfonlayer}{nodelayer}
		\node [style=bn] (24) at (-0.75, 0) {};
		\node [style=none] (25) at (0.75, 0) {};
	\end{pgfonlayer}
	\begin{pgfonlayer}{edgelayer}
		\draw (24) to (25.center);
	\end{pgfonlayer}
\end{tikzpicture}
$ otherwise.
	
	We study the composition $\syn{\bullet}\xrightarrow{!_{\mor{\g}}} \syn{\mor{\g}}\xrightarrow{m} \syn{\g}$. 
	By definition, $!_{\mor{\g}}$ sends the unique generator $I \to \bullet$ to $\compcomp{\Sigma_{\mor{\g}}}$.
	Since the moralisation connects parents, $\lbrace v \rbrace \cup \parents{v}$ is a clique in $\mor{\g}$ for every $v\in V_{\g}$, i.e.\ all $\input{graph_parents_morph.tikz}$ are in the image of $m$.
	Consequently, $m(\compcomp{\Sigma_{\mor{\g}}})=\compcomp{\graph{\Sigma_{\g}}}$, because the compare morphisms are preserved by hypergraph functors. 
	By \cref{lem:graphcopycomp}, $\compcomp{\graph{\Sigma_{\g}}}= \graph{\copycomp{\Sigma_{\g}}}$.
	Moreover, $\copycomp{\Sigma_{\g}}$ has trivial input, so it coincides with its graph:  $\compcomp{\graph{\Sigma_{\g}}}= \copycomp{\Sigma_{\g}}$.
	On the other side, $\tilde{!}_{\g}\colon \syn{\bullet}\to \syn{\g}$ sends $I\to \bullet$ to $\copycomp{\Sigma_{\g}}$, so the wanted factorisation holds.

	As discussed near \eqref{eq:factoringdistribution1}, this was the only non-straightforward step required to ensure that $(\omega,\g,\tau)\in \bn$ implies $(\omega,\mor{\g},\tau)\in \mn$. In particular, note that the construction leaves the assignments unchanged.
	
	We are left to discuss how the moralisation acts on morphisms. 
	Let $\alpha \colon \g' \to \g$ be an order-preserving graph homomorphism. 
	Then $\alpha$ can be seen as an order-preserving graph homomorphism $\mor{\g'}\to \mor{\g}$. 
	This is easily checked: every edge $v\edge w$ in $\mor{\g'}$ either comes from an edge of $\g'$, so it is obviosly preserved by $\alpha$, or it is given by a certain vertex $u$ such that $v\to u$ and $w\to u$. Since $\alpha$ sends edges to edges, $\alpha(v)$ and $\alpha(w)$ are parents (or coincide with) $\alpha(u)$. 
	This means that we have an edge $\alpha(v)\edge \alpha(w)$ in $\mor{\g}$ (or $\alpha(v)=\alpha(w)$).
	We can then define the moralisation functor on morphisms by sending $(\alpha,\eta,\hat{\eta})\colon (\omega,\g,\tau)\to (\omega',\g',\tau)$ to $(\alpha,\eta,\hat{\eta})\colon (\omega,\mor{\g},\tau)\to (\omega',\mor{\g'},\tau')$.
\end{proof}

\section{Triangulation as a Functor}\label{sec:triangulation}
Our next aim is to define a functor $\mn \to \bn$. This direction, however, is more subtle and it requires a modification of the semantics. The difficulty stems from the fact that moving from $\mat$ to $\finstoch$ is not as immediate as the converse operation.

From a computability perspective, it is sensible to separate this direction in two steps: one that is completely syntactic and another one that is completely semantic. To this aim, we discuss chordal networks.

\begin{definition}
	An \newterm{ordered chordal graph} is an ordered DAG $\g$ such that, for every three vertices $u,v,w$, the following implication holds:
	\begin{equation*}
	u,v\to w\text{ and }u\le v\qquad \implies \qquad u\to v.
	\end{equation*}
\end{definition}
Every ordered chordal graph is a chordal graph (see \cite[Defs.~2.24 and 2.25]{koller2009probabilistic}), but not every chordal graph is ordered --- this distinction depends on the vertex ordering. The notion of ordered chordality may seem ad hoc, but it plays a practical role in computation, as we will see more explicitly when discussing the variable elimination algorithm. Indeed, while this terminology may be nonstandard, such a vertex order is typically fixed to ensure algorithmic correctness.

We also emphasize that chordal graphs are often times considered as undirected. 
However, in our context the directed counterpart is preferable, as at the level of syntax we need less generators. In particular, this nicely connects to cluster graphs and the junction tree algorithm, as we will investigate in~\cref{rem:JTA}.

\begin{definition}\label{def:ubn_cat}
	The \newterm{category of (irredundant) chordal networks} $\cn$ is defined as:
	\begin{itemize}
		\item The objects are triples $(\omega , \g,\tau)$, where $\omega$ is a CD-functor $\cdsyn{\bullet} \to \finstoch$, $\g$ is ordered chordal, and $\tau\colon \vcat{\g}\to \finstoch$ is a symmetric monoidal functor for which there exists $F\colon \syn{\g}\to \mat$ such that the following commutes: 
		\begin{equation*}
		\begin{tikzcd}
			&&\finstoch \ar[drr,"i", bend left=10]& &\\ 
			\cdsyn{\bullet}\ar[urr,"\omega", bend left=10]\ar[rd,"\star"] &&\vcat{\g}\ar[d,hook]\ar[u,"\tau"]&& \finprojstoch\\
			&\syn{\bullet}\ar[r,"\tilde{!}_{\g}"] &\syn{\g}\ar[r,dashed,"F"] & \mat\ar[ur,"q"]& 
		\end{tikzcd}
	\end{equation*}
		\item A morphism $(\omega,\g,\tau)\to (\omega',\g',\tau')$ is a triple $(\alpha,\eta,\hat{\eta})$, where $\alpha$ is a morphism $\alpha\colon \g' \to \g$  in $\odag$, $\eta$ is a monoidal transformation $\tau \to \tau' \vcat{\alpha}$, $\hat{\eta}\colon \omega \to \omega'$ is a distribution-preserving monoidal transformation, and moreover $\eta$ and $\hat{\eta}$ coincide as transformations between functors $\vcat{\bullet}\to \finstoch$; more precisely, the following commutes: 
\begin{equation*}
\begin{tikzcd}[row sep=large]
\vcat{\bullet}\ar[r]\ar[d,hook] & \vcat{\g} \ar[d,bend right=40, "\tau" {name=T,left}]\ar[d,bend left=40, "\tau'\comp \vcat{\alpha}" {name=Tprime,right}]\\
\cdsyn{\bullet}\ar[r,bend right=40,"\omega" {name=W,below},end anchor=205]\ar[r,bend left=40, "\omega'" {name=Wprime,above},end anchor=150] & \finstoch
\arrow[from=W, to=Wprime, Rightarrow, shorten=2mm,"\hat{\eta}"] \arrow[from=T, to=Tprime, Rightarrow, shorten=2mm, "\eta"]
\end{tikzcd}
\end{equation*}
	Composition is component-wise.
	\end{itemize}	
\end{definition}
In other words, we have used the definition of $\mn$, but considered for special DAGs. Note that the following holds without necessarily passing from $\bn$.
\begin{proposition}\label{prop:cmoralisation}
We have a functor $\cmor{-}\colon \cn \to \mn$ simply given by precomposition with $m\colon \syn{\mor{\g}} \to \syn{\g}$, defined as in \eqref{eq:moralisation_functor}.
\end{proposition}
The additional $C$ in $\cmor{-}$ is meant to emphasise the chordal variant of the functor. The proof is a simple adaptation of the proof of \cref{thm:moralisation}.
We now aim at defining a converse functor $\ctr{-}\colon \mn \to \cn$.

\begin{definition}
	Let $\h$ be an ordered undirected graph. Its \newterm{triangulation} $\tr{\h}$ is the ordered directed graph where $v \to w$ whenever $v \le w$ and there is a path $v \edge w_1 \edge \dots \edge w_n = w$ with $w_i \geq w$ for each $i=1, \dots, n$.
\end{definition}
In particular, every edge $v\edge w$ in $\h$ yields a directed edge in $\tr{\h}$. The idea is that $\tr{\h}$ only enforces the conditional independencies that we know to hold thanks to $\h$.

By direct check, $\tr{\h}$ is an ordered chordal graph. For clarity, we reserve the notation $\tr{\h}$ when considering $\bn$, and $\ctr{\h}$ when considering $\cn$ (the additional $C$ standing for chordal).
This mirrors the distinction between $\mor{-}\colon \bn \to \mn$ and $\cmor{-}\colon \cn \to \mn$.

As in the case of moralisation, we want to capture triangulation as a functor $\mn \to \cn$. The construction is similar: the key step is defining the functor $t\colon \syn{\ctr{\h}} \to \syn{\h}$ which will act by precomposition on the given syntax category $\syn{\h}$ of the Markov network $\h$. It is defined on the generators of $\syn{\h}$ by the following clause
\begin{equation}\label{eq:triangulation}
	{\input{parents_morph.tikz}}\qquad \longmapsto \qquad 
	{\input{triangulation_def.tikz}}
\end{equation}
where $C_v\coloneqq \lbrace \phi \colon I\to C \mid v\in C\subseteq \parents{v} \cup \lbrace v \rbrace\rbrace \subseteq \Sigma_{\h}$, $P\coloneqq \parents{v} \cap C_v$ and $Q\coloneqq \parents{v} \setminus P$.
Intuitively, the right-hand side compares all the factors involving $v$ and its parents, while $Q$ accounts for the remaining inputs.

\begin{example}\label{ex:misconception_triangulation}
The triangulation of the undirected graph $\h$ of \cref{ex:misconception} is the DAG obtained by making all edges direct and adding an edge $A\to C$ (because of the path $A\edge D \edge C$). 
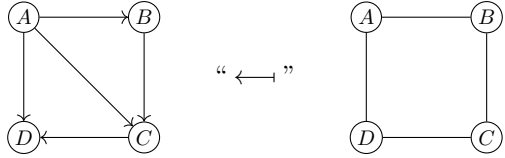
\[
\input{undirected_graph_triangulation.tikz} \qquad \text{``}\longmapsfrom\text{''} \qquad 	\begin{tikzpicture}
	\begin{pgfonlayer}{nodelayer}
		\node [style=bw] (15) at (-2, 2) {$A$};
		\node [style=bw] (18) at (-2, -2) {$D$};
		\node [style=bw] (29) at (2, 2) {$B$};
		\node [style=bw] (31) at (2, -2) {$C$};
	\end{pgfonlayer}
	\begin{pgfonlayer}{edgelayer}
		\draw (15) to (18);
		\draw (29) to (31);
		\draw (31) to (18);
		\draw (29) to (15);
	\end{pgfonlayer}
\end{tikzpicture}

\]
The string diagram representing $\tr{\h}$ in $\syn{\tr{\h}}$ is given below left, with its image under the functor $t\colon \syn{\tr{\h}} \to \syn{\h}$ below right: 
\begin{equation*}
	\scalebox{0.89}{\input{misconception_triangulation.tikz}}
\end{equation*}
The equations in \eqref{eq:comonoids} and \eqref{eq:hypergraph} ensure that the right hand side correspond to the string diagram representing $\syn{\h}$ (see \cref{fig:markov_network_ex}). Now, the factors associated with the original network, defined in~\eqref{tab:misconception_factors}, yield a hypergraph functor $\Phi \colon \syn{\h}\to \mat$.
Postcomposing $t$ with $\Phi$ sets the following $\mat$-semantics for $\syn{\tr{\h}}$-generators: $f_D (D \given AC) \coloneqq \phi_D(D) \phi_{AD}(AD) \phi_{CD}(CD)$, $f_C (C \given AB) \coloneqq \phi_{BC}(BC) \phi_C(C)$, $f_B (B \given A)\coloneqq \phi_B(B) \phi_{AB}(AB)$ and $f_A(A)\coloneqq \phi_A(A)$. (For simplicity, here we use the same notation for the generators in the syntax categories and their image in $\mat$). 
Observe that, by definition, $f_C(C\given AB)$ does not really depend on $A$, but this additional input is forced by the triangulation. 
This plays an important role in the discussion of \cref{sec:VE}, where chordality will be crucial to ensure functoriality (see \cref{ex:misconception_triangulation_2}). 
\end{example}

\begin{theorem}\label{thm:triangulation}
	Triangulation gives rise to a functor $\ctr{-}\colon \mn \to \cn$  which on objects maps $(\omega,\h,\tau)$ to $(\omega,\ctr{\h},\tau)$.
\end{theorem}
\begin{proof}
	We start by showing that $\syn{\bullet}\to\syn{\h}$ factors through $t\colon \syn{\ctr{\h}}\to \syn{\h}$ as defined in~\eqref{eq:triangulation}. 

	By definition, $\tilde{!}_{\ctr{\h}}\colon \syn{\bullet}\to \syn{\ctr{\h}}$ sends the unique generator $I\to\bullet$ to the copy-composition $\copycomp{\Sigma_{\ctr{\h}}}$.
	As we noted in the previous proof, $\copycomp{\Sigma_{\ctr{\h}}}=\graph{\copycomp{\Sigma_{\ctr{\h}}}}$ because it has trivial input. Therefore, by \cref{lem:graphcopycomp}, $\copycomp{\Sigma_{\ctr{\h}}} = \compcomp{\graph{\Sigma_{\ctr{\h}}}}$. 
	The image of this morphism via $\syn{\ctr{\h}}\to \syn{\h}$ is then given by $\compcomp{\bigcup_{v \in V_{\ctr{\h}}}\compcomp{C_v}}$, where $C_v 
	= \lbrace \phi \colon I\to C \mid v\in C\subseteq \parents{v} \cup \lbrace v \rbrace\rbrace \subseteq \Sigma_{\h} $.
	Here, the additional tensoring seemingly needed with $\begin{tikzpicture}
	\begin{pgfonlayer}{nodelayer}
		\node [style=bn] (24) at (-0.75, 0) {};
		\node [style=none] (25) at (0.75, 0) {};
		\node [style=none] (26) at (1.25, 0) {$Q$};
	\end{pgfonlayer}
	\begin{pgfonlayer}{edgelayer}
		\draw (24) to (25.center);
	\end{pgfonlayer}
\end{tikzpicture}
$ (where $Q$ is defined as in~\eqref{eq:triangulation}) can be disregarded because for every vertex $v$, the fact that $\lbrace v \rbrace\in C_v$ ensures that $v$ is always considered in $\compcomp{\bigcup_{v \in V_{\ctr{\h}}}\compcomp{C_v}}$.
	Moreover, $\compcomp{\bigcup_{v \in V_{\ctr{\h}}}\compcomp{C_v}} = \compcomp{\bigcup_{v \in V_{\ctr{\h}}} C_v}$ by \cref{lem:compcomp} because the sets $C_v$ are all disjoint: if a clique $D$ belongs to $C_v\cap C_w$, then $v,w \in D$ and both of them are parents of each other, which contradicts the fact that $\ctr{\h}$ is a DAG.

	We claim that $\bigcup_{v \in V_{\ctr{\h}}}{C_v}=\Sigma_{\h}$, where $\subseteq$ holds by definition. 
	Let us write $\bigcup_{v \in V_{\ctr{\h}}}{C_v}$ more explicitly:
	\[
		\bigcup_{v \in V_{\ctr{\h}}}{C_v}= \left\lbrace \phi\colon I \to C \given \exists\, v\in V_{\ctr{\h}} \text{ such that }v\in C\subseteq \parents{v}\cup\lbrace v\rbrace \right\rbrace.
	\]
	Now, given any clique $C\in \clique{\h}$, let us consider the biggest element $v\in C$. 
	Then for any element $w \in C$, there must exist an edge $w \edge v$ because $C$ is a clique.
	By assumption, $w\le v$, and therefore $w\to v$ in $\ctr{\h}$. 
	We conclude that $v \in C\subseteq \parents{v}\cup\lbrace v\rbrace$, so $\phi \colon I \to C$ belongs to $C_v$.  
	As claimed, $\bigcup_{v \in V_{\ctr{\h}}}{C_v}=\Sigma_{\h}$, and therefore 
	$\syn{\bullet}\to \syn{\ctr{\h}}\to \syn{\h}$ sends $I\to \bullet$ to $\compcomp{\Sigma_{\h}}$.
	Since $!_{\h}\colon \syn{\bullet}\to \syn{\h}$ is defined by the same assignment, we have shown that $!_{\h}$ indeed factorises through $t$.
	Finally, any $(\omega,\h,\tau)\in \mn$ leads to $(\omega, \ctr{\h},\tau)\in \cn$ in the following way:
	\[
	\begin{tikzcd}[column sep=large]
		\cdsyn{\bullet}\ar[r,"\tilde{!}_{\ctr{\h}}"]\ar[rr,bend left=20,start anchor=60, "!_{\h}"]\ar[d,"\omega"]&\syn{\ctr{\h}} \ar[r,"t"]\ar[rd,dashed]&  \syn{\h}\ar[d,"\Phi", dashed]\\
		\finstoch\ar[r,"i"] & \finprojstoch & \mat\ar[l,"q" above]
	\end{tikzcd}
	\]
	As for moralisation, the construction above does not affect the assignment $\tau$, which is therefore preserved.

	Regarding morphisms, $(\alpha,\eta,\hat{\eta}) \colon (\omega,\h,\tau)\to (\omega',\h',\tau')$ can be reinterpreted in the new type $(\omega,\ctr{\h},\tau)\to (\omega',\ctr{\h'},\tau')$. 
	The only nontrivial part is that $\alpha$ is indeed an ordered graph homorphism $\ctr{\h'}\to \ctr{\h}$. This follows by definition: whenever $v\to w$ in $\tr{\h'}$, then there is a path $v\edge w_1 \edge \dots \edge w_n =w$ in $\h'$, with $w_i \ge w$. Applying $\alpha \colon \h'\to \h$, we obtain a path $\alpha(v)\edge \alpha(w_1) \edge \dots \edge \alpha(w_n)=\alpha(w)$ with $\alpha(w_i)\ge \alpha(w)$ (as edges may contract, this path is possibly shorter than the one in $\h'$). Therefore, $\alpha(v)\to \alpha(w)$, as wanted.
\end{proof}

\begin{proposition}\label{prop:chordalcmor_id}
	The composition $\ctr{\cmor{-}}\colon \cn \to \cn$ coincides with the identity endofunctor $\id_{\cn}$.
\end{proposition}
In particular, starting with a Bayesian network, we can apply the functor $\ctr{\mor{-}}$ to obtain its associated chordal representative, and this representative remains fixed under moralisation and triangulation. See \cref{sec:interplay} for a complete picture of the functorial interplay.
\begin{proof}
	The only thing to check is that the graphs are preserved by the endofunctor. 
	To see why that is the case, take an ordered chordal graph $\g$. 
	By definition, $\cmor{\g}$ is not adding additional edges: Indeed, whenever $u,v \to w$, i.e.\ whenever two edges are parents of the same vertex, then either $u\le v$, and thus $u\to v$, or $v \le u$, and thus $v\to u$.

	We are now reduced to show that $\ctr{\cmor{\g}}$ also does not add any edge to $\cmor{\g}$. 
	By definition, $v \to w$ in $\ctr{\cmor{\g}}$ holds if and only if $v \edge w_1 \edge \dots \edge w_n \edge w$ with $w_i \ge w\ge v$.
	We proceed by induction on the length $n$ of the path, noting that the base case $n=0$ is immediate.
	For the induction step, since both $v$ and $w$ are smaller than every $w_i$, the maximum element of the path must be some $w_k$. Therefore, $w_{k-1} \edge w_k \edge w_{k+1}$ (where $w_{k-1}=v$ if $k=1$ and $w_{k+1}=w$ if $k=n$) implies that $w_{k-1},w_{k+1}\to w_{k}$ in $\g$ because $\cmor{\g}$ does not add new edges. Since $\g$ is chordal, we have either $w_{k-1}\to w_{k+1}$ or $w_{k+1}\to w_{k-1}$, and both leads to an edge $w_{k-1}\edge w_{k+1}$ in $\cmor{\g}$. The path can therefore be shorten by one element, and by induction hypothesis, we conclude that $v \edge w$ in $\cmor{\g}$, as wanted.

	We therefore conclude $v \to w$ in $\ctr{\cmor{\g}}$ if and only if $v \edge w$ in $\cmor{\g}$ with $v\le w$, and this condition holds if and only if $v\to w$ in $\g$.
	As wanted, $\ctr{\cmor{\g}}=\g$.
\end{proof}

\begin{remark}[Cluster Graphs and the Junction Tree Algorithm]\label{rem:JTA}
	In this remark, we briefly discuss how cluster graphs and the junction tree algorithm can be considered in the categorical formalism developed in this paper.
	A standard example for a \emph{cluster graph}\footnote{As the precise notion is not of primary interest here, we defer its formal definition to \cite[Def.~10.1]{koller2009probabilistic}.} associated to an ordered chordal graph $\g$ can be defined as follows: 
	\begin{itemize}
		\item Vertices are $C_v = \parents{v}\cup \lbrace v \rbrace$, called \emph{clusters}.
		\item $C_v \edge C_w$ if $C_w\cap C_v \neq \emptyset$, and in that case they are labelled by $C_w \cap C_v$.
	\end{itemize}
	Then the junction tree algorithm consists in choosing a spanning tree $\mathcal{T}$ in the cluster graph that satisfies the \emph{running intersection property}: if $v \in V_{\g}$ belongs to two clusters $C$ and $D$, then it also belongs to every cluster in the (unique) path in $\mathcal{T}$ that connects $C$ and $D$. 
	This property can be interpreted from a categorical perspective as requiring the tree to connect each occurrence of the same vertex, as it is standard when considering $\input{copy.tikz}$ and $\input{compare.tikz}$.
	To avoid redundancy, the tree is typically simplified to include only maximal cliques. This is sensible, since chordal graphs allow efficient maximal clique computation, unlike general graphs where the problem is NP-hard~\cite[Sec.~10.4.2]{koller2009probabilistic}.
	\begin{figure}[!t]
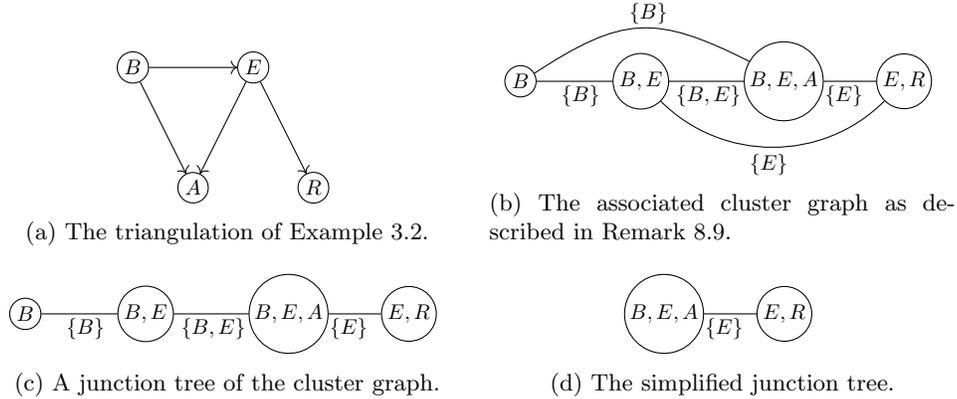
\centering
	\begin{subfigure}[b]{0.4\textwidth}
		\centering
	{\input{BEAR_triangulation.tikz}}
	\subcaption{The triangulation of \cref{ex:BEAR}.} 
	\end{subfigure}
	\hspace{1ex}
	\begin{subfigure}[b]{0.4\textwidth}
		\centering
	{\input{BEAR_cluster.tikz}}
	\subcaption{The associated cluster graph as described in \cref{rem:JTA}.} 
	\end{subfigure}
	\par \vspace{2ex}
	\begin{subfigure}[b]{0.4\textwidth}
		\centering
	{\input{BEAR_tree.tikz}}
	\subcaption{A junction tree of the cluster graph.} 
	\end{subfigure}
		\hspace{1ex}
	\begin{subfigure}[b]{0.4\textwidth}
		\centering
	{\input{BEAR_treesimple.tikz}}
	\subcaption{The simplified junction tree.} 
	\end{subfigure}
		\caption{A visual example of cluster graphs and junction trees.}\label{fig:clusterBEAR}
	\end{figure}
	
	For a better understanding, we provide a visual example in \cref{fig:clusterBEAR}, where we use the triangulation of the BEAR Bayesian network of \cref{ex:BEAR}. 
	It is worth noting that all the graphical representations are not changing the string diagrammatic situation except for the last one, where the cluster graph $\lbrace B,E,A\rbrace$ is interpreted as the string diagram
	\[
	\input{BEAR_JTsimple.tikz}
	\] 
	This is a direct consequence of merging the two cliques $\lbrace B \rbrace$ and $\lbrace B,E \rbrace$ into $\lbrace B,E,A \rbrace$. 

	To understand the key differences between the chordalization, the cluster graph, and the junction tree, we need to consider how the latter two are used; namely, in message passing algorithms. Through this lens, also edges give rise to a generator of the syntax. Explicitly, the junction tree in \cref{fig:clusterBEAR} will have the following generators, where we used the notation of \cite[Sec.~10.3]{koller2009probabilistic}:
	\[
	\input{BEAR_tree_gen.tikz}
	\]
	(The first four generators are associated to clusters, while the last three describe messages between the clusters).
	In particular, we have generators of the same type, in contrast with the syntaxes studied in this paper.
	Although an algebraic approach to message passing is a potentially interesting direction, giving a proper account of the algorithm would require the ability to perform arithmetic division --- an operation not provided by the categorical framework. Additionally, this issue appears to be semantic in nature, at least from the perspective adopted here.
	For these reasons, we leave a possible formal treatment to future work.
\end{remark}

\section{Variable Elimination as a Functor}\label{sec:VE}
We now aim to describe a functor $\cn\to \bn$, which is strongly connected to Variable Elimination, and as such we will call it $\ve{-}$.
Contrary to $\mor{-}$, $\cmor{-}$ and $\ctr{-}$, this functorial description is totally semantic.
When we compose $\ve{-}$ with $\ctr{-}$, we obtain what is known as \emph{triangulation}, which is a process that turns a Markov network into a Bayesian one.

\begin{theorem}[Functorial Variable Elimination]\label{thm:triangulation_finstoch}
	Let $\g$ be an ordered chordal graph.
	Let $\cC$ be a hypergraph category with total conditionals and a normalisation cospan $\cC_{\tot} \xrightarrow{i} \cD \xleftarrow{q} \cC$ (\cref{set:conditionals}).
	If the following diagram 
	\[
	\begin{tikzcd}
		\cdsyn{\bullet}\ar[r,"!_{\g}"]\ar[d,"\omega"]&\cdsyn{\g} \ar[r,"\star_{\g}"]& \syn{\g}\ar[d,"F"]\\
		\cC_{\tot} \ar[r,"i"] & \cD & \cC \ar[l,"q" above]
	\end{tikzcd}
	\]
	commutes, where $F$ is a hypergraph functor, then there exists a factorisation $\cdsyn{\bullet}\xrightarrow{!_{\g}} \cdsyn{\g} \xrightarrow{G} \cC_{\tot}$ of $\omega$ for some CD-functor $G$.

	In particular, we obtain a functor $\ve{-}\colon \cn \to \bn$ which acts as the identity on both objects and morphisms.
\end{theorem}

\begin{proof}
	The statement can be translated as follows: Let $\g$ be an ordered chordal graph and consider a state $\omega \colon I \to V_{\g}$ and a factorisation $i(\omega)=q(\copycomp{\lbrace f_v\rbrace })$, where $f_v= F(\input{parents_morph.tikz})$ belongs to $\cC$. Then $\omega$ admits a factorisation $\omega = \copycomp{\lbrace g_v \rbrace}$ where $g_v \in \cC_{\tot}$. 

	We proceed by induction. Let us assume that the statement holds true for all graphs with $n-1$ vertices, and prove it for graphs with $n$ vertices.
	
	Let us take the biggest element $v\in V_{\g}$. Consider $\lambda_v \coloneqq \begin{tikzpicture}
	\begin{pgfonlayer}{nodelayer}
		\node [style=morphism] (54) at (0, 0) {$f_v$};
		\node [style=none] (55) at (-2, 0) {};
		\node [style=bn] (58) at (2, 0) {};
		\node [style=none] (59) at (0, 0) {};
	\end{pgfonlayer}
	\begin{pgfonlayer}{edgelayer}
		\draw (55.center) to (54);
		\draw (59.center) to (58);
	\end{pgfonlayer}
\end{tikzpicture}
$ and $g_v \coloneqq \begin{tikzpicture}
	\begin{pgfonlayer}{nodelayer}
		\node [style=morphism] (54) at (0, 0) {$\normalisation{f_v}$};
		\node [style=none] (55) at (-2, 0) {};
		\node [style=none] (58) at (2, 0) {};
		\node [style=none] (59) at (0, 0) {};
	\end{pgfonlayer}
	\begin{pgfonlayer}{edgelayer}
		\draw (55.center) to (54);
		\draw (59.center) to (58.center);
	\end{pgfonlayer}
\end{tikzpicture}
\in \cC_{\tot}$. We note that, by definition,
	\[
	\begin{tikzpicture}
	\begin{pgfonlayer}{nodelayer}
		\node [style=morphism] (54) at (0, 0) {$f_v$};
		\node [style=none] (55) at (-2, 0) {};
		\node [style=none] (58) at (2, 0) {};
		\node [style=none] (59) at (0, 0) {};
	\end{pgfonlayer}
	\begin{pgfonlayer}{edgelayer}
		\draw (55.center) to (54);
		\draw (59.center) to (58.center);
	\end{pgfonlayer}
\end{tikzpicture}
\quad =\quad \input{normalisation_ve.tikz}
	\] 
	By assumption, if we take the biggest parent $w$ of $v$, then $\parents{v}\setminus\lbrace w \rbrace \subseteq \parents{w}$.
	In particular, we can define
	\begin{gather*}
		f'_w \coloneqq \, \input{fprime_w.tikz}\, ,\qquad \text{and} \\
		f'_u\coloneqq f_u \qquad \text{for all }u\in V_{\g}\setminus \lbrace v,w\rbrace
	\end{gather*}
	Let $\g'$ be the graph obtained from $\g$ by deleting $v$. From the factorisation $i(\omega)= q(\copycomp{\lbrace f'_u \rbrace \cup \lbrace g_v \rbrace })$, and a direct check shows that $q(\copycomp{\lbrace f'_u \rbrace})$ is equal to the marginalization $i(\omega')$ of $i(\omega)$.
	Since $\g'$ is ordered chordal because $\g$ is, we can apply the induction hypothesis to show that $\omega'$ can be factorised in $\cC_{\tot}$.
	Since $g_v\in \cC_{\tot}$, the statement is also true for $\omega$ because $i(\omega)= q(\copycomp{\lbrace g_u\rbrace }) = i (\copycomp{\lbrace g_u\rbrace })$ and $i$ is faithful.

	The last sentence now follow because we have ensured that any $(\omega,\g,\tau)\in \cn$ actually allows for a factorisation through $\finstoch$, i.e.\ $(\omega,\g,\tau) \in \bn$.
\end{proof}

\begin{remark}\label{rem:v-structures}
	The trivial definition of $\ve{-}$ may lead one to question why the restriction to ordered chordal graphs is necessary in \cref{thm:triangulation_finstoch}. 
	Indeed, although the proof relies heavily on chordality, this alone does not rule out the possibility that the theorem could extend to all ordered DAGs.
	To address this, we provide a counterexample demonstrating that such a generalization does not hold.
	
	Let $\g$ be the DAG given by $\begin{tikzpicture}
	\begin{pgfonlayer}{nodelayer}
		\node [style=bw] (15) at (0, 0) {$A$};
		\node [style=bw] (29) at (2.5, 0) {$C$};
		\node [style=bw] (31) at (5, 0) {$B$};
	\end{pgfonlayer}
	\begin{pgfonlayer}{edgelayer}
		\draw [style=arrow] (15) to (29);
		\draw [style=arrow] (31) to (29);
	\end{pgfonlayer}
\end{tikzpicture}
$.
	Any Bayesian network over $\g$ makes $A$ and $B$ independent of each other once we discard $C$, i.e.\ $\input{phi_noC.tikz}=\input{indep.tikz}$ for any $\input{phi_ABC.tikz}$ associated to a Bayesian network over $\g$.
	We claim that this independence is not necessarily true for a hypergraph functor $\syn{\g}\to \mat$, thus the restriction to ordered chordal graphs is sensible.
	To this end, consider $A=B=C=\lbrace 0,1 \rbrace$, and let $\input{phi_ABC.tikz}$ be the morphism defined by $\phi(a,b,c)=1$ if exactly two of them are equal and 0 otherwise.
	We then define $\Phi \colon \syn{\g}\to \mat$ by sending the generators $I\to A$ and $I\to B$ to $\begin{tikzpicture}
	\begin{pgfonlayer}{nodelayer}
		\node [style=bn] (24) at (-0.75, 0) {};
		\node [style=none] (25) at (0.75, 0) {};
	\end{pgfonlayer}
	\begin{pgfonlayer}{edgelayer}
		\draw (24) to (25.center);
	\end{pgfonlayer}
\end{tikzpicture}
$ and $A\tensor B \to C$ to \input{phiC.AB.tikz}. 
	In this way, $\phi$ corresponds to the distribution $\cdsyn{\bullet}\rightarrow \syn{\g}\xrightarrow{\Phi} \mat$.
	By direct computations, $\input{phi_noC.tikz}\neq \input{indep.tikz}$, so indeed $A$ and $B$ share some dependence even when $C$ is discarded.
\end{remark}

\begin{remark}[Variable Elimination]\label{rem:ve}
	The proof of \cref{thm:triangulation_finstoch} mirrors the variable elimination algorithm, whose goal is to compute the marginal of the smallest element. More precisely, we derive a simplified version of the standard algorithm by restricting it to ordered chordal graphs (see \cref{algo:ve}). The key difference from the standard algorithm for general Markov networks is that the procedure \textsc{Sum-Product-Eliminate-Var} in~\cite[Alg.~9.1]{koller2009probabilistic} --- which marginalises the selected variable $\tau(v)$ after performing a compare-composition of all factors having $v$ as output --- can be captured as a combination of \textsc{Marg-Var} and $\ctr{-}$. 
\begin{algorithm}[!t]
	\caption{Variable Elimination for Ordered Chordal Graphs}\label{algo:ve}
		\hrulefill

		\noindent \textbf{Procedure} \textsc{VE} (

		\hspace{2ex} $\g$ $\backslash\backslash$ Ordered chordal graph

		\hspace{2ex} $F$ $\backslash\backslash$ Family of morphisms given by $\syn{\g}\to \mat$

		)
		
		\hspace{2ex} Let $v_1, \dots, v_n$ be an ordering of $V_{\g}$ such that

		\hspace{4ex} $v_i < v_j$ if and only if $i<j$

		\hspace{2ex} \textbf{for} $i=n,n-1,\dots, 1$

		\hspace{4ex} ${F}\leftarrow $ \textsc{Marg-Var}($F$,$v_i$)

		\hspace{2ex} \textbf{endfor} 
		
		\hspace{2ex} $\omega' \leftarrow \copycomp{F}$

		\hspace{2ex} \textbf{Return} $\omega'$

		\textbf{endProcedure}
		\par\vspace{2ex}
		\textbf{Procedure} \textsc{Marg-Var} (
			
		\hspace{2ex} $F$ $\backslash\backslash$ Family of morphisms given by $\syn{\g}\to \mat$
		
		\hspace{2ex} $v$ $\backslash\backslash$ Vertex corresponding to the variable to be marginalised
		
		)
		
		\hspace{2ex} $\lambda_v \leftarrow \begin{tikzpicture}
	\begin{pgfonlayer}{nodelayer}
		\node [style=morphism] (54) at (0, 0) {$f_v$};
		\node [style=none] (55) at (-2, 0) {};
		\node [style=bn] (58) at (2, 0) {};
		\node [style=none] (59) at (0, 0) {};
	\end{pgfonlayer}
	\begin{pgfonlayer}{edgelayer}
		\draw (55.center) to (54);
		\draw (59.center) to (58);
	\end{pgfonlayer}
\end{tikzpicture}
$
		
		\hspace{2ex} $\tilde{F} \leftarrow F \setminus \lbrace f_{v}\rbrace$

		\hspace{2ex} $w \leftarrow \max \lbrace u \in \In{f_v} \rbrace$

		\hspace{2ex} $\tilde{F}' \leftarrow \tilde{F} \setminus \lbrace f_w \rbrace$

		\hspace{2ex} $f'_w \leftarrow \input{fprime_w.tikz}$

		\hspace{2ex} \textbf{Return} $\tilde{F}'\cup \lbrace f'_w \rbrace$

		\textbf{endProcedure}

\end{algorithm}
We also note that the validity of \cref{algo:ve} follows directly from the proof of \cref{thm:triangulation_finstoch}, with the only difference being that the algorithm does not require defining $g_v$.
\end{remark}

\begin{example}\label{ex:misconception_triangulation_2}
	By applying the proof of \cref{thm:triangulation_finstoch} to \cref{ex:misconception_triangulation}, and recalling the values of the factors given in \cref{ex:misconception}, we obtain the following
\begin{equation*}
{\begin{tabular}{|c|c|c|}
\hline
$A$ & $C$ & $g_D(d\given AC)$\\
\hline \hline
$a$ & $c$ & 0.5\\
\hline
$a$&$\no{c}$& 0.9999\\
\hline
$\no{a}$&$c$& 0.0001\\
\hline
$\no{a}$&$\no{c}$&0.5\\
\hline
\end{tabular}}
\hspace{0.3cm}
{\begin{tabular}{|c|c|c|}
\hline
$A$ & $B$ & $g_C(c\given AB)$\\
\hline \hline
$a$ & $b$ & 0.6666\\
\hline
$a$&$\no{b}$& 0.0002\\
\hline
$\no{a}$&$b$& 0.9998\\
\hline
$\no{a}$&$\no{b}$&0.3334\\
\hline
\end{tabular}}
\hspace{0.3cm}
{\begin{tabular}{|c|c|}
\hline
$A$ & $g_B(b\given A)$\\
\hline \hline
$a$ & 0.2307 \\
\hline
$\no{a}$& 0.8475\\
\hline
\end{tabular}}
\hspace{0.3cm}
{\begin{tabular}{|c|}
\hline
$g_A(a)$\\
\hline \hline
0.1806 \\
\hline
\end{tabular}}
\end{equation*}
In particular, we observe that the values of $A$ do influence $c$, reinforcing the point made earlier in \cref{rem:v-structures}.
\end{example}

\begin{proposition}
	$\ve{-}\colon \cn \to \bn$ is fully faithful.
\end{proposition}
Despite its full faithfulness, the $\ve{-}$ embedding depends on the graph structure and does not admit straightforward extensions, as discussed in \cref{rem:v-structures}.
\begin{proof}
Since $\ve{-}$ acts as the identity on both objects and morphisms, this follows from our definition of the morphisms in both $\cn$ and $\bn$.
\end{proof}

\section{Functorial Interplay}\label{sec:interplay} 

We now investigate the interactions between the functors introduced in the previous sections.
To start, we discuss the following commutative diagram, already appeared in the introduction as \eqref{eq:interplay_intro}:
\begin{equation}\label{eq:interplay}
\begin{tikzcd}[column sep=large]
	&\bn \ar[r,"\mor{-}"]\ar[rrr, bend left=20, start anchor=north east, "\tr{\mor{-}}"] & \mn \ar[rr, color=red, "\tr{-}"]\ar[rd,"\ctr{-}" below left] && \bn\\
	\cn \ar[ru, color=red, "\ve{-}"]\ar[rru,"\cmor{-}" below right]\ar[rrr,equal] &&&\cn\ar[ru, color=red, "\ve{-}"] &
\end{tikzcd}
\end{equation}
where only the red arrows --- $\ve{-}$ and $\tr{-}$ --- require semantic assumptions, while the black ones only consider syntax.
From previous sections, we have already discussed
\begin{itemize}
	\item The definitions of $\mor{-}$, $\cmor{-}$, $\ctr{-}$, and $\ve{-}$ (\cref{thm:moralisation,prop:cmoralisation,thm:triangulation,thm:triangulation_finstoch});
	\item The fact that $\ctr{\cmor{-}}=\id_{\cn}$ (\cref{prop:chordalcmor_id}).
\end{itemize}
Moreover, we define $\tr{-}\coloneqq \ve{\ctr{-}}\colon \mn \to \bn$, which perfectly reflects the definition given in \cite{lorenzin2025moralisation}, so the commutativity on the left is obvious. 
We are therefore left to show that $\cmor{-}$ coincides with $\mor{\ve{-}}$ and that $\tr{\mor{-}}$ is purely syntactic, so for example we have no need to require the existence of conditionals.
The former result is a direct consequence of the definitions --- recall that $\ve{-}$ acts as the identity on both objects and morphisms --- while the latter follows from the fact that $\tr{\mor{-}}$ just adds additional inputs to boxes, i.e.\ $\tr{\mor{-}}$ is obtained by $\syn{\tr{\mor{\g}}}\to \syn{\g}$ defined on generators as 
\[
\input{parents_morph.tikz} \quad \longmapsto \quad \input{mortr.tikz}
\]
where $\parents{v}$ indicates the parents of $v$ in $\tr{\mor{\g}}$, $P$ its parents in $\g$, and $Q\coloneqq \parents{v} \setminus P$. 
In particular, $\tr{\mor{-}}$ is highly more efficient than computing it as a composition. 

\begin{remark}
	With $\tr{\mor{-}}$, the variable elimination algorithm reduces to just taking any associated $F\colon \cdsyn{\g}\to \finstoch$ and consider the state associated to the smallest element.
	This conclusion can be directly drawn from \cref{algo:ve}, since $\lambda_v= \begin{tikzpicture}
	\begin{pgfonlayer}{nodelayer}
		\node [style=bn] (35) at (0.75, 0) {};
		\node [style=none] (36) at (-0.75, 0) {};
	\end{pgfonlayer}
	\begin{pgfonlayer}{edgelayer}
		\draw (35) to (36.center);
	\end{pgfonlayer}
\end{tikzpicture}
$ holds in every call to the procedure \textsc{Marg-Var}, and therefore we can just as well not change $f_w$ and simply return $\tilde{F}$.
	Note this follows because the order of each DAG is \emph{topological}, i.e.\ it respects the edges (see \cref{rem:ordering}). 

	To apply this functorial framework to more meaningful situations, we should allow a change of order --- otherwise, we can only compute the marginal of the smallest element.
	A first remark is that such an operation cannot happen at the level of directed graphs, since orders need to respect the structure, and this would rarely be the case. 
	Moreover, CD-categories do not allow \begin{tikzpicture}
	\begin{pgfonlayer}{nodelayer}
		\node [style=none] (37) at (-3.25, 0) {};
		\node [style=none] (38) at (-2, -0.75) {};
		\node [style=none] (39) at (-2, 0.75) {};
	\end{pgfonlayer}
	\begin{pgfonlayer}{edgelayer}
		\draw [in=-90, out=180] (38.center) to (37.center);
		\draw [in=180, out=90] (37.center) to (39.center);
	\end{pgfonlayer}
\end{tikzpicture}
 and \begin{tikzpicture}
	\begin{pgfonlayer}{nodelayer}
		\node [style=none] (29) at (-2, 0) {};
		\node [style=none] (30) at (-3.25, 0.75) {};
		\node [style=none] (31) at (-3.25, -0.75) {};
	\end{pgfonlayer}
	\begin{pgfonlayer}{edgelayer}
		\draw [in=90, out=0] (30.center) to (29.center);
		\draw [in=0, out=-90] (29.center) to (31.center);
	\end{pgfonlayer}
\end{tikzpicture}
, so that inputs cannot become outputs.
	The situation is different for undirected graphs, where in particular a change of order can be described by an equivalence $\sigma \colon \syn{\h} \xrightarrow{\cong} \syn{\h}$ that simply permutes the order of the outputs for each generator. 
	To better illustrate the situation, let us consider the BEAR network of \cref{ex:BEAR}, and the order $A<B<E<R$ for $\mor{\g}$. 
	Then the resulting triangulation becomes
	\[
	\input{BEAR_triangulation_orderswap.tikz} \qquad\qquad \input{BEAR_triangulation_orderswap_stringdiagram.tikz}
	\] 
	This can be linked to the starting Bayesian network using the functors $m$ and $t$ defined in \eqref{eq:moralisation_functor} and \eqref{eq:triangulation} respectively, by means of the composition $\syn{\tr{\mor{\g}}} \xrightarrow{t} \syn{\mor{\g}}\xrightarrow{\sigma} \syn{\mor{\g}} \xrightarrow{m} \syn{\g}$, where $\sigma$ is associated to the change of order from $B<E<A<R$ to $A<B<E<R$. This results in the following inductive description on generators: 
	\[
	\input{bear_triang_aber.tikz}
	\]
\end{remark}

We now turn our attention to natural transformations between the functors. 

\begin{proposition}\label{prop:trmor_mortr}
	There are two natural transformations: $\tr{\mor{-}} \to \id_{\bn}$ and $\mor{\tr{-}} \to \id_{\mn}$, where $\id_{-}$ indicates the identity functor.
	Moreover, $\tr{\mor{-}}$ and $\mor{\tr{-}}$ are idempotent endofunctors.
\end{proposition}
In particular, $\tr{\mor{-}}$ and $\mor{\tr{-}}$ are idempotent semimonads.
\begin{proof}
	The two natural transformations are simply achieved by noting that the identity on an ordered DAG $\g$ and an ordered undirected graph $\h$ are also morphisms $\g\to \tr{\mor{\g}}$ and $\h\to\mor{\tr{\h}}$. 
	That these correspond to natural transformation follows because the functors $\mor{-}$ and $\tr{-}$ send the morphism $(\alpha,\eta,\hat{\eta})$ to the same pair intepreted in the new type, as stated at the end of the proofs of \cref{thm:moralisation,thm:triangulation}.

	Idempotency follows from \eqref{eq:interplay}: indeed, the commutative diagram 
	\[ 
		\begin{tikzcd}[column sep=large]
		\mn\ar[rr,"\tr{-}"]\ar[dr,"\ctr{-}" below left] &&\bn \ar[r,"\mor{-}"] & \mn \ar[rr,"\tr{-}"]\ar[rd,"\ctr{-}" below left] && \bn\\
		&\cn \ar[ru, "\ve{-}"]\ar[rru,"\cmor{-}" below right]\ar[rrr,equal] &&&\cn\ar[ru, "\ve{-}"] &
	\end{tikzcd}
	\]
	shows that $\tr{\mor{\tr{-}}} =\ve{\ctr{-}}=\tr{-}$.
\end{proof}

Finally, the reader may wonder whether an adjunction may be achieved, but this is too much to ask as it would require to have either natural $\id_{\bn}\to \tr{\mor{-}}$ or $\id_{\mn}\to \mor{\tr{-}}$.
	However, this is not possible already at the level of graphs, because in general $\tr{\mor{\g}}$ (resp.\ $\mor{\tr{\h}}$) has more edges than $\g$ (resp.\ $\h$), and represents less independencies.

	\begin{proposition}\label{prop:noadjunction}
		There is no adjunction given by $\mor{-}$ and $\tr{-}$.
	\end{proposition}
	\begin{proof}
		By contradiction, let us assume that a natural transformation $\mu_{(\omega,\g,\tau)}\colon (\omega,\g,\tau)\to (\omega,\tr{\mor{\g}},\tau)$ exists. 
		Recall that the functors $\mor{-}$ and $\tr{-}$ are `identity-like', i.e.\ they send $(\alpha,\eta,\hat{\eta})$ to $(\alpha,\eta,\hat{\eta})$ intepreted in the new type. 
		For any $(\omega,\g,\tau)$ and any vertex $v \in \g$, let us consider $(\omega_v,\bullet,\tau_v)$, where $\omega_v$ is the marginalization of $\omega$ at $v$, and $\tau_v\colon \vcat{\bullet}\to \finstoch$ is defined by setting $\tau_v(\bullet)\coloneqq \tau(v)$.
		The marginalization gives rise to a distribution-preserving monoidal transformation $\hat{\pi}\colon \omega \to \omega_v$, so in particular we have a morphism $(i,\pi,\hat{\pi})\colon (\omega,\g,\tau)\to (\omega_v,\bullet,\tau_v)$ where $i$ is the inclusion of graphs $\bullet \to \g$ sending $\bullet$ to $v$ and $\pi \colon \tau \to \tau_v \vcat{i}$ is the identity on $v$ and $\begin{tikzpicture}
	\begin{pgfonlayer}{nodelayer}
		\node [style=bn] (35) at (0.75, 0) {};
		\node [style=none] (36) at (-0.75, 0) {};
	\end{pgfonlayer}
	\begin{pgfonlayer}{edgelayer}
		\draw (35) to (36.center);
	\end{pgfonlayer}
\end{tikzpicture}
$ on the other vertices.
		By assumption,
		\[
		\begin{tikzcd}[column sep=large]
			(\omega,\g,\tau )\arrow[d,"{(i,\pi,\hat{\pi})}"]\ar[r,"\mu_{(\omega,\g,\tau)}"] & (\omega,\tr{\mor{\g}},\tau)\arrow[d,"{(i,\pi,\hat{\pi})}"]\\
			(\omega_v,\bullet,\tau_v)\ar[r,"\mu_{(\omega_v,\bullet,\tau_v)}"] & (\omega_v,\bullet,\tau_v)
		\end{tikzcd}
		\]
		commutes.
		At the level of graphs, $\mu_{(\omega_v,\bullet,\tau_v)}$ must be given by the identity, as it is the only morphism $\bullet \to \bullet$. 
		Therefore, whatever $\alpha \colon \tr{\mor{\g}}\to \g$ represents $\mu_{(\omega,\g,\tau)}$ at the level of graphs, it must respect the commutation of the diagram above, which means that $\alpha(v)=v$.
		The arbitrariety of $v$ implies that $\alpha$ must be the identity, but the identity is not necessarily a graph homomorphisms $\tr{\mor{\g}}\to \g$ because $\tr{\mor{\g}}$ has in general more edges than $\g$.
		For the sake of an example, let $\g\coloneqq \begin{tikzpicture}
	\begin{pgfonlayer}{nodelayer}
		\node [style=bw] (15) at (0, 0) {$A$};
		\node [style=bw] (29) at (2.5, 0) {$C$};
		\node [style=bw] (31) at (5, 0) {$B$};
	\end{pgfonlayer}
	\begin{pgfonlayer}{edgelayer}
		\draw [style=arrow] (15) to (29);
		\draw [style=arrow] (31) to (29);
	\end{pgfonlayer}
\end{tikzpicture}
$, and note that $\tr{\mor{\g}}$ is a complete DAG (i.e.\ we have the additional edge $A \to B$). 
	
		The same idea applies when assuming the existence of a natural transformation $(\omega,\h,\tau)\to (\omega,\mor{\tr{\h}},\tau)$, and therefore $\mor{-}$ and $\tr{-}$ do not admit a unit for the possible adjunction in either direction. 
	\end{proof}

	\section{Conclusions}\label{sec:conclusions}
	We provided a categorical framework for Bayesian and Markov networks, focusing on translations between them in the form of moralisation and triangulation. In the spirit of Lawvere's categorical semantics, we characterised networks as functors from a `syntax' (the graph) to a `semantics' (its probabilistic interpretation). The translations were formulated abstractly by functor pre-composition, and defined inductively on the diagrammatic syntax.
	Triangulation required semantic assumptions that allows a procedure strongly reminiscent of the variable elimination algorithm.
	We moreover offered a preliminary discussion on the Junction Tree Algorithm (\cref{rem:JTA}).
	Overall, this approach ensures a study where semantics is considered only when explicitly needed; recall \eqref{eq:interplay_intro} and \eqref{eq:interplay}, providing a diagrammatic overview of the studied transformations.

A first direction for future work is accounting for additional PGMs. 
Our separation of syntax and semantics offers a clear way forward characterising several existing models, including Gaussian networks (where a discrete probabilistic semantics is replaced with a Gaussian one), networks based on `partial DAGs' (where the syntax allows both for directed and undirected edges), hidden Markov models (where not all vertices of the graph are `visible' to the string diagram interfaces), and factor graphs (where the syntax describes bipartite graphs). It also abstracts and simplifies reasoning on translations between models, whose specification in the existing literature is often in natural language and prone to ambiguity.

A second direction is to apply our abstract account of moralisation and triangulation to study in more detail the junction tree algorithm, with the goal of expressing clique tree message passing and graph-based optimisation. 
Ultimately, the aim is to offer a compositional perspective on these algorithms, leveraging our syntactic description of moralisation and triangulation.

{\small
\bibliography{references}}

\end{document}